\documentclass{article} 
\usepackage{iclr2027_conference,times}

\usepackage{amsmath,amsfonts,bm}

\def\eqref#1{equation~\ref{#1}}

\def\1{\bm{1}}

\DeclareMathAlphabet{\mathsfit}{\encodingdefault}{\sfdefault}{m}{sl}
\SetMathAlphabet{\mathsfit}{bold}{\encodingdefault}{\sfdefault}{bx}{n}

\usepackage{hyperref}
\usepackage{url}
\usepackage[utf8]{inputenc} 
\usepackage[T1]{fontenc}    
\usepackage{hyperref}       
\usepackage{url}            
\usepackage{booktabs}       
\usepackage{amsfonts}       
\usepackage{nicefrac}       
\usepackage{microtype}      
\usepackage[dvipsnames]{xcolor}         
\usepackage[nointegrals]{wasysym}
\usepackage{graphicx}
\usepackage{wrapfig}
\usepackage{makecell}
\usepackage{multirow}
\usepackage{subcaption}
\usepackage{amsmath}
\usepackage{amssymb}

\usepackage[ruled,vlined]{algorithm2e}

\usepackage{amsmath,amssymb,amsthm}
\newtheorem{theorem}{Theorem}
\newtheorem{lemma}{Lemma}

\newtheorem{corollary}{Corollary}

\usepackage{enumitem}

\everypar{\looseness=-1}

\title{Simultaneous Neural Optimal Transport}

\author{Milena Gazdieva \thanks{Corresponding author} \\
Applied AI Institute\\
AXXX \\
Moscow, Russia\\
\texttt{milenagazdieva@gmail.com} \\
\And
Kirill Sokolov \\
Lomonosov Moscow State University \\
Moscow, Russia \\
\AND
Jiawei Chen\\
MIRIAI\thanks{Moscow Independent Research Institute of Artificial Intelligence}\\
Moscow, Russia
\And
Evgeny Burnaev \\
Applied AI Institute\\
AXXX \\
Moscow, Russia
\And
Alexander Korotin\\
Applied AI Institute\\
AXXX \\
Moscow, Russia
}

\iclrfinalcopy 

\begin{document}

\maketitle

\begin{abstract}
Optimal Transport (OT) provides a principled framework for learning transformations between probability distributions from unpaired samples. In many applications, however, a single transformation must map several source distributions to a common target distribution.
For example, image restoration might require handling different types of degradation without knowing the degradation of each input at inference time. Simple approaches of pooling the source distributions only encourage alignment with the target at the aggregate level and may leave individual sources misaligned. In our paper, we consider the simultaneous OT problem which formalizes the task of learning a shared transport map that minimizes the average transport cost while aligning each source distribution with a prescribed target. 
We propose a neural method for solving the simultaneous OT problem by learning a shared transport map that minimizes the average transport cost while aligning each source distribution with a prescribed target. We derive a max-min formulation for learning this map. We illustrate its application to image restoration, where a single model handles multiple degradation types using a common collection of clean target images.
\end{abstract}

\begin{figure*}[h]
    \centering
    \includegraphics[width=\textwidth]{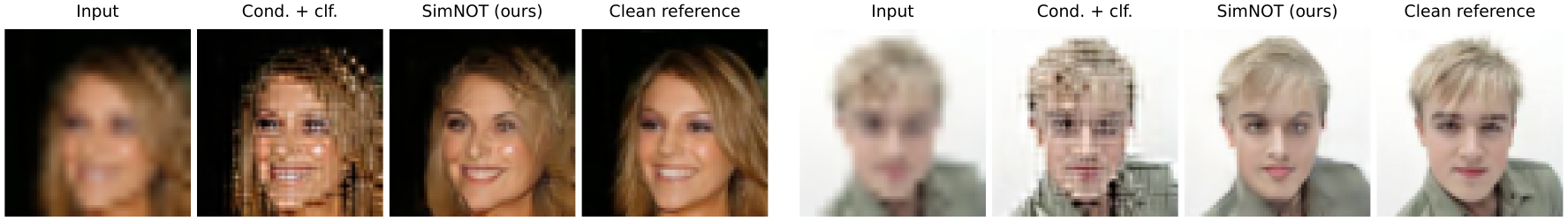}
    \caption{
        Restoration of selected CelebA $64\!\times\! 64$ test images under
        bilinear downsampling with an unseen factor
        ($\times 3$ at test time versus $\times 4$ during training).
        Each group shows the degraded input, conditional OT
        with classifier-based routing, SimNOT (ours),
        and the clean reference.
    }
    \label{fig:teaser}
\end{figure*}

\section{Introduction}
\label{sec:introduction}

Learning transformations between probability distributions from unpaired samples underlies a range of applications, including image-to-image translation~\citep{zhu2017unpaired}, voice conversion~\citep{chun2023non}, and prediction of single-cell perturbation responses~\citep{bunne2023learning}.
In many practical scenarios, the input data originate from several distributions, while a single model is expected to transform them into a common target distribution. The challenge is to learn one transformation that produces suitable outputs for each source distribution while preserving relevant properties of the inputs.
For example, all-in-one image restoration aims to handle multiple types of degradation within one model without requiring the degradation type at inference time~\citep{potlapalli2023promptir}.
In the unpaired setting, the learner has access to degraded and clean images but not to their correspondences.

Optimal Transport (OT) provides a mathematical framework for learning such transformations~\citep{peyre2019computational}.
It seeks a transport map or plan that aligns the source and target distributions while minimizing a prescribed transport cost.
The cost specifies which input--output correspondences are preferable and can be chosen to encourage preservation of particular attributes.
Neural OT methods make it possible to approximate transport maps between distributions available only through samples and apply the learned maps to previously unseen inputs~\citep{fan2023neural, korotin2023kernel,korotin2023neural}.
These methods provide a natural starting point for learning transformations shared across several source distributions.

A straightforward approach is to pool the source datasets and learn an OT map from their mixture to the target distribution.
However, matching the transformed mixture to the target does not guarantee that the transformed distribution of each source matches the target individually.
For example, different sources may be mapped to different parts of the target distribution, producing a correct aggregate distribution despite substantial discrepancies for individual sources.
Alternatively, learning a separate map for each source does not yield a single transformation that can be applied without selecting a source-specific model.
This motivates learning a shared map while explicitly accounting for each source distribution.

The \textit{simultaneous optimal transport} (SOT) framework~\citep{wang2022simultaneous} formalizes transport under multiple distributional constraints using a common map or stochastic kernel.
We consider its many-to-one setting: a single transformation must align several source distributions with a prescribed target while minimizing the average transport cost.
The shared transformation couples the source-specific transport tasks, and the individual alignment requirements retain information that is lost when the sources are treated only as a mixture.
Our focus is on the computational setting in which the distributions are unknown and available through unpaired samples, and the learned transformation must generalize to new inputs.

To develop a practical learning method, we consider an \textit{unbalanced} formulation of simultaneous transport, replacing hard marginal constraints with divergence penalties.
Unbalanced OT has already been used to learn generative models and transport maps between a pair of distributions~\citep{choi2024generative,gazdieva2024light, yang2018scalable, klein2024genot, eyring2024unbalancedness}.
In our setting, this relaxation provides an adjustable trade-off between transport cost and marginal agreement while retaining the shared transport structure and separate objectives for the sources.
We derive a variational training objective and use it to learn a single neural transport map together with source-specific potentials.
The source identity is used to select the appropriate potential during training, whereas the transport map itself is shared and does not require this identity at inference time.

\noindent\textbf{Contributions.}
We develop \textbf{Simultaneous Neural Optimal Transport
(SimNOT)}, a continuous solver for learning a shared
transport map from multiple source distributions to
a prescribed target using unpaired samples.
We introduce a divergence-based unbalanced formulation
of simultaneous OT, establish its exact semi-dual
representation (\wasyparagraph\ref{sec:objective}), and provide neural approximation guarantees (\wasyparagraph\ref{sec:theoretical_properties}).
These results yield a neural solver with one shared map
and source-specific potentials, requiring no source labels
at inference time (§\ref{sec:training}).
We demonstrate its effectiveness on simultaneous
Gaussian-to-Swiss-roll transport and unpaired image
restoration with multiple degradation types, including
evaluation under unseen degradation strengths
(§\ref{sec:experiments}).

\noindent\textbf{Notation.}
We take the source and target domains to be nonempty closed sets
$\mathcal X\subseteq\mathbb R^{D_x}$ and $\mathcal Y\subseteq\mathbb R^{D_y}$. We write $\mathcal Z$ for an auxiliary noise space.
For a Borel space $E$, $\mathcal P(E)$ and $\mathcal M_+(E)$ denote the sets of Borel probability measures and finite non-negative Borel measures on $E$, respectively, and $C_b(E)$ denotes the bounded continuous real-valued functions on $E$.
The source distributions are denoted by $\mathbb P_1,\ldots,\mathbb P_K\in\mathcal P(\mathcal X)$, and the common target distribution by $\mathbb P^\ast\in\mathcal P(\mathcal Y)$.
For a measurable map $T$, $T_\#\mu$ denotes the pushforward of a measure $\mu$.
We represent a stochastic map by $T:\mathcal X\times\mathcal Z\to\mathcal Y$, where $z\sim\mathbb S\in\mathcal P(\mathcal Z)$ is independent of the input; its output distribution for source $\mathbb P_k$ is $T_\#(\mathbb P_k\otimes\mathbb S)$.
For a probability kernel $\gamma(\cdot\mid x)$ from $\mathcal X$ to $\mathcal Y$, we write $\gamma_\#\mu$ for the output measure $\int \gamma(\cdot\mid x)\,d\mu(x)$.
For $\gamma\in\mathcal M_+(\mathcal X\times\mathcal Y)$, its marginals are denoted by $\gamma_x$ and $\gamma_y$.
We use $\Pi(\mu,\nu)$ for the set of finite non-negative measures on the product space with marginals $\mu$ and $\nu$ (when these marginals have the same total mass); in particular, for probability marginals these are probability couplings.

\section{Related Work}
\label{sec:related_work}

We review the most closely related OT formulations and
solvers below. \textit{\underline{Extended related work}} on OT barycenters and
all-in-one image restoration is discussed in
Appendix~\ref{app:related_work}.

\textbf{Learning optimal transport maps.}
OT has been used in generative modeling both as a discrepancy
between generated and target distributions and as a principle
for learning the transformation itself.
For example, Wasserstein GANs~\citep{arjovsky2017wasserstein}
use an OT discrepancy as a training loss, but do not require
the generator to minimize the transport cost between its inputs
and outputs.
More closely related to our work are continuous neural OT
solvers, which learn transport maps or plans from
samples~\citep{makkuva2020optimal,rout2022generative,
korotin2023neural}.
These methods exploit OT duality to construct optimization
objectives for neural potentials and transport maps, with
applications to generation and unpaired translation.
Partial transport formulations also permit relaxing distribution
matching while controlling the similarity between inputs and
outputs~\citep{gazdieva2023extremal}.
These works consider transport between a pair of distributions;
our focus is on learning one transformation jointly for several
sources.

\textbf{Simultaneous optimal transport.}
\citet{wang2022simultaneous} introduce simultaneous OT as
transport between vector-valued measures: a common map or
transport kernel must transport several source measures to
their respective targets.
They study Monge and Kantorovich formulations, existence,
and duality.
The balanced problem with a common target considered in our
work is a particular case of this framework.
Its defining requirement is that the same transformation
satisfy the transport constraints for every source.
Learning independent pairwise OT maps does not impose this
requirement, while matching a pooled source distribution to
the target does not ensure that each source is matched
individually.
We build on this perspective to develop a neural method for
learning shared transformations from samples.
We also consider divergence-based marginal relaxation;
this differs from the inequality-constrained unbalanced
formulation studied by \citet{wang2022simultaneous}.

\textbf{Unbalanced optimal transport.}
Unbalanced OT replaces exact marginal constraints with
penalties for deviations from the prescribed
measures~\citep{liero2018optimal}.
Among continuous solvers, \citet{yang2018scalable}
propose an adversarial approach that jointly learns
a transport map and a source mass scaling function.
UOTM~\citep{choi2024generative} learns a transport
network and a potential through a semi-dual objective.
The latter approach is particularly relevant to
the optimization principles used in our work.
Both methods consider transport between two measures.
Our method combines marginal relaxation with a shared
transformation and source-specific potentials.
The role of the unbalanced formulation here is to provide
a flexible transport objective for simultaneous learning.

\section{Background}
\label{sec:background}

In this section, we recall the OT formulations relevant to our
work and specify the simultaneous OT problem with a
common target. We then introduce the marginal relaxation
considered in this paper. For background on OT and unbalanced
OT, we refer to \citet{villani2008optimal} and
\citet{chizat2017unbalanced}, respectively.

\subsection{Optimal Transport}
\label{sec:background_ot}

Let $\mathbb P\in\mathcal P(\mathcal X)$,
$\mathbb Q\in\mathcal P(\mathcal Y)$ be probability measures,
$c:\mathcal X\times\mathcal Y\to\mathbb R_+$
$-$ continuous function.

\textbf{Monge formulation.}
The Monge OT problem seeks a measurable map
$T:\mathcal X\to\mathcal Y$ that transports $\mathbb P$
to $\mathbb Q$ while minimizing the transport cost given by function $c(\cdot,\cdot)$:
\begin{equation}
\label{eq:monge}
\inf_{T:\,T_{\#}\mathbb P=\mathbb Q}
\int_{\mathcal X}c(x,T(x))\,d\mathbb P(x).
\end{equation}
A minimizer $T^\ast$ of~\eqref{eq:monge},
when it exists, is called an optimal transport (OT) map.

\textbf{Kantorovich formulation.}
Monge OT problem may admit no feasible map or no minimizer.
Thus, it is common to consider its Kantorovich relaxation which allows splitting the mass of
each input point:
\begin{equation}
\label{eq:kantorovich}
\operatorname{OT}_c(\mathbb P,\mathbb Q)
=
\inf_{\pi\in\Pi(\mathbb P,\mathbb Q)}
\int_{\mathcal X\times\mathcal Y}c(x,y)\,d\pi(x,y).
\end{equation}
Here $\Pi(\mathbb P,\mathbb Q)$ is the set of probability
measures on $\mathcal X\times\mathcal Y$ with marginals
$\pi_x=\mathbb P$ and $\pi_y=\mathbb Q$.
Under the assumptions of this paper, a minimizer $\pi^\ast$
of~\eqref{eq:kantorovich} always exists, but need not be
unique~\citep{villani2008optimal}.
Such a minimizer is called an optimal transport plan.

Disintegration gives $d\pi(x,y)=d\pi_x(x)\,d\pi(y\mid x)$,
where $\pi(\cdot\mid x)$ is the conditional distribution
of outputs for an input $x$.
An OT plan $\pi^\ast$ thus defines a stochastic transport map
$x\mapsto\pi^\ast(\cdot\mid x)$.
If $\pi^\ast(\cdot\mid x)=\delta_{T^\ast(x)}$,
then $\pi^\ast=(\mathrm{id}_{\mathcal X},T^\ast)_{\#}\mathbb P$
and $T^\ast$ is an OT map.

\textbf{Unbalanced formulation.}
Unbalanced OT (UOT) relaxes the marginal constraints
in~\eqref{eq:kantorovich} by penalizing deviations of the
marginals from $\mathbb P$ and $\mathbb Q$.
For finite non-negative measures $\mu,\nu$ on the same
space, we define the $\psi$-divergence as
\begin{equation}
\label{eq:divergence}
D_\psi(\mu\|\nu)
=
\begin{cases}
\displaystyle
\int \psi\!\left(\frac{d\mu}{d\nu}\right)\,d\nu,
& \mu\ll\nu,\\[1ex]
+\infty,
& \text{otherwise}.
\end{cases}
\end{equation}
Here $\psi:[0,\infty)\to[0,\infty]$ is convex,
lower semicontinuous, vanishes only at $1$, and satisfies
$\lim_{u\to\infty}\psi(u)/u=+\infty$.
We define $D_\phi$ analogously.
Examples include the Kullback--Leibler and $\chi^2$
divergences, generated by $\psi(u)=u\log u-u+1$
and $\psi(u)=(u-1)^2$, respectively.

The UOT problem is given by~\citep{liero2018optimal}
\begin{equation}
\label{eq:uot}
\begin{aligned}
\operatorname{UOT}_{c,\psi,\phi}(\mathbb P,\mathbb Q)
=
\inf_{\gamma\in\mathcal M_+(\mathcal X\times\mathcal Y)}
\Bigg[
&\int_{\mathcal X\times\mathcal Y}c(x,y)\,d\gamma(x,y)
+D_\psi(\gamma_x\|\mathbb P)
+D_\phi(\gamma_y\|\mathbb Q)
\Bigg],
\end{aligned}
\end{equation}
where $\mathcal M_+(\mathcal X\times\mathcal Y)$ is the set
of finite non-negative measures, and $\gamma_x,\gamma_y$
are the marginals of $\gamma$.
A minimizer $\gamma^\ast$ is called an unbalanced OT plan.
Its total mass is not constrained to equal one.
The strength of the marginal penalties can be adjusted
by scaling the generators $\psi$ and $\phi$.

The balanced problem~\eqref{eq:kantorovich} is recovered
when $\psi$ and $\phi$ are the convex indicators of $\{1\}$:
the divergence terms then enforce
$\gamma_x=\mathbb P$ and $\gamma_y=\mathbb Q$.

As in the balanced case, disintegration gives
$d\gamma(x,y)=d\gamma_x(x)\,d\gamma(y\mid x)$.
The conditional distributions $\gamma^\ast(\cdot\mid x)$
define a stochastic transport map acting on the source
marginal $\gamma_x^\ast$, which may differ from $\mathbb P$.
In the deterministic case,
$\gamma^\ast(\cdot\mid x)=\delta_{T^\ast(x)}$
and $\gamma_y^\ast=T^\ast_{\#}\gamma_x^\ast$.

\subsection{Simultaneous Optimal Transport}
\label{sec:background_sot}

\begin{wrapfigure}{r}{0.4\linewidth}
    \centering\vspace{-6mm}
    \includegraphics[width=1.2\linewidth]{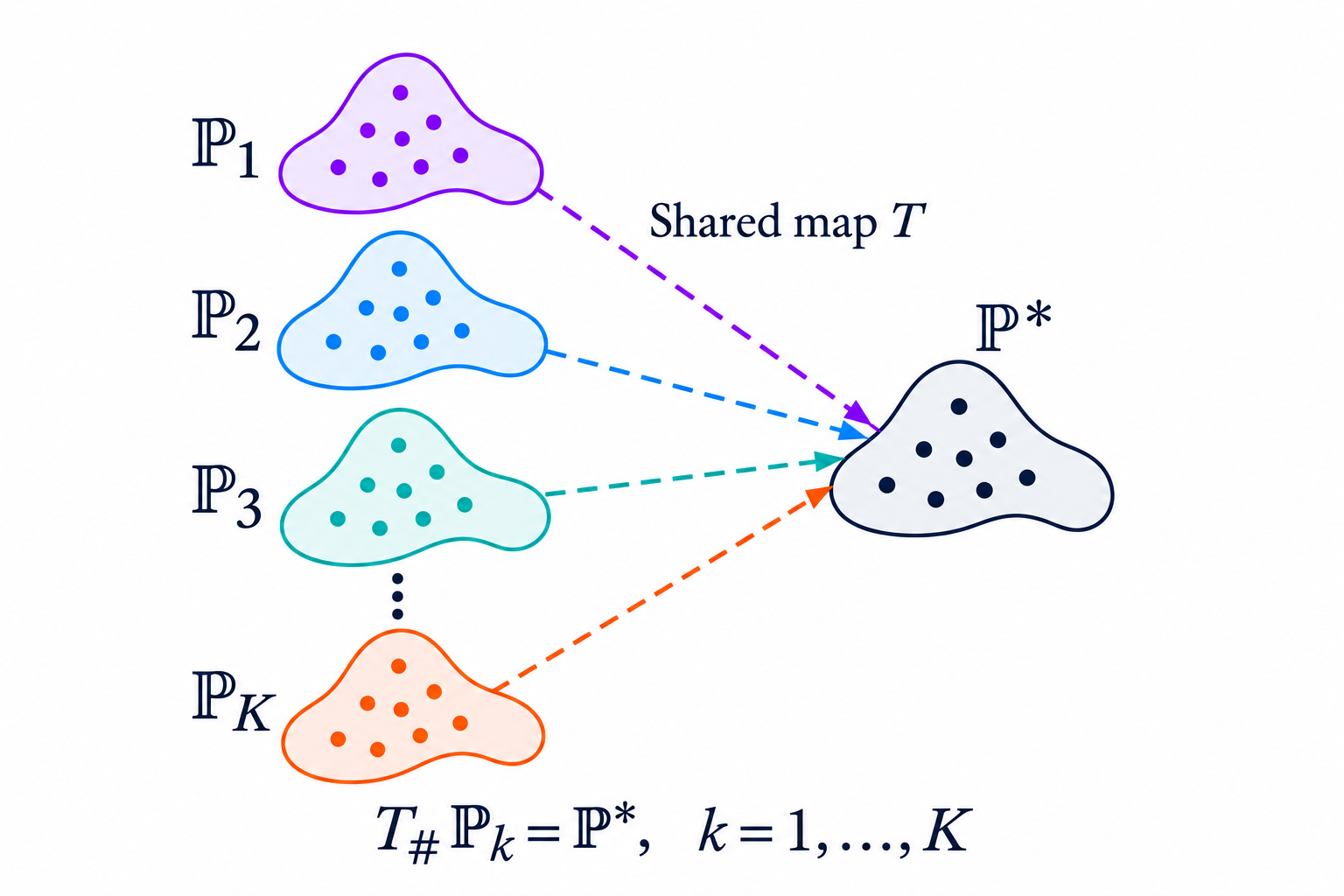}
    \vspace{-6mm}\caption{\centering A schematic illustration of\\ simultaneous OT.}
    \label{fig:simultaneous_ot}
\end{wrapfigure}

Simultaneous OT seeks a shared transport rule that transports
several source measures to their respective target
measures~\citep{wang2022simultaneous}.
Here, we consider the case of $K$ source distributions
$\mathbb P_1,\dots,\mathbb P_K\in\mathcal P(\mathcal X)$
and a common prescribed target distribution
$\mathbb P^\ast\in\mathcal P(\mathcal Y)$.

\textbf{Balanced formulation.}
A simultaneous transport plan from
$\mathbb P_1,\dots,\mathbb P_K$ to $\mathbb P^\ast$
is represented by a family of plans
$\pi_k\in\Pi(\mathbb P_k,\mathbb P^\ast)$
sharing a common conditional distribution
$\pi(\cdot\mid x)$:
\begin{equation}
\label{eq:sot_shared_conditional}
\!\!\!\!\!d\pi_k(x,y)=d\mathbb P_k(x)\,d\pi(y\mid x),
\; k=1,\dots,K.
\end{equation}
Thus, the same conditional distribution of outputs
is used for a given input $x$ across all sources.
The balanced simultaneous OT problem seeks such
a family minimizing the average transport cost: 
\vspace{-4mm}\begin{equation}
\label{eq:sot_balanced}
\inf_{\substack{
\pi(\cdot\mid x):\\
(\pi_k)_y=\mathbb P^\ast,\; k=1,\dots,K
}}
\frac{1}{K}\sum_{k=1}^K
\int_{\mathcal X\times\mathcal Y}
c(x,y)\,d\pi_k(x,y).
\end{equation}
Here, the infimum is taken over shared conditional
probability distributions, with the plans $\pi_k$
defined by~\eqref{eq:sot_shared_conditional}. A minimizing family $\{\pi_k^\ast\}_{k=1}^K$,
when it exists, is called a simultaneous optimal
transport (OT) plan.
Its shared conditional distribution
$\pi^\ast(\cdot\mid x)$ defines a stochastic
transport map $x\mapsto\pi^\ast(\cdot\mid x)$
that transports every $\mathbb P_k$
to $\mathbb P^\ast$.
In the special case of deterministic plans,
$\pi(\cdot\mid x)=\delta_{T(x)}$ for a shared map
$T:\mathcal X\to\mathcal Y$ satisfying
$T_{\#}\mathbb P_k=\mathbb P^\ast$ for all $k$, see Fig. \ref{fig:simultaneous_ot}.
Optimizing~\eqref{eq:sot_balanced} over such plans
yields the simultaneous Monge problem.

When the source distributions have pairwise disjoint
supports, simultaneous OT reduces to independent
pairwise OT problems:
the corresponding transport maps, when they exist,
can be combined into a single map.
For overlapping sources, independently obtained maps
may disagree on the overlap, whereas simultaneous OT
requires the same transformation for every source.

\textbf{Unbalanced formulation.}
We consider an unbalanced relaxation of
\eqref{eq:sot_balanced} that penalizes deviations
from the prescribed marginals while retaining
a shared conditional distribution.
Specifically, we consider finite nonnegative
transport plans $\gamma_1,\dots,\gamma_K$ of the form
\begin{equation}
\label{eq:suot_shared_conditional}
d\gamma_k(x,y)
=
d(\gamma_k)_x(x)\,d\gamma(y\mid x),
\qquad k=1,\dots,K,
\end{equation}
where $\gamma(\cdot\mid x)$ is a shared conditional
probability distribution.
The source marginals $(\gamma_k)_x$ are optimized
jointly with this conditional distribution,
and the plans $\gamma_k$ need not have unit mass.

The resulting simultaneous UOT problem is
\begin{equation}
\label{eq:sot_unbalanced}
\begin{aligned}
&\inf_{\gamma(\cdot\mid x),\,\{(\gamma_k)_x\}_{k=1}^K}
\frac{1}{K}\sum_{k=1}^K
\Bigg[
\int_{\mathcal X\times\mathcal Y}
c(x,y)\,d\gamma_k(x,y)
+
D_\psi\!\left(
(\gamma_k)_x\middle\|\mathbb P_k
\right)
+
D_\phi\!\left(
(\gamma_k)_y\middle\|\mathbb P^\ast
\right)
\Bigg].
\end{aligned}
\end{equation}
Here, the infimum is taken over shared conditional
probability distributions and finite nonnegative
source marginals, with the plans $\gamma_k$
defined by~\eqref{eq:suot_shared_conditional}.
The divergences penalize deviations of the source
and target marginals from $\mathbb P_k$ and
$\mathbb P^\ast$, respectively.

A minimizing family $\{\gamma_k^\ast\}_{k=1}^K$,
when it exists, is called a simultaneous UOT plan.
Its shared conditional distribution
$\gamma^\ast(\cdot\mid x)$ defines the corresponding
stochastic transport map. In the case of deterministic plans,
$\gamma(\cdot\mid x)=\delta_{T(x)}$ for a shared map
$T:\mathcal X\to\mathcal Y$, and
$(\gamma_k)_y=T_{\#}(\gamma_k)_x$ for all $k$.

To the best of our knowledge, this divergence-based
formulation of simultaneous unbalanced OT has not
been studied previously.
The unbalanced formulation of~\citet{wang2022simultaneous}
uses marginal domination constraints.

\subsection{Computational Setup}
\label{sec:computational_setup}

The distributions $\mathbb P_1,\dots,\mathbb P_K$
and $\mathbb P^\ast$ are unknown and accessible only
through finite unpaired samples.
Our goal is to learn a shared transport map for
the underlying distributions that admits new inputs
not present in the training data.
This setup is typically called
continuous~\citep{fan2023neural,rout2022generative,korotin2023kernel}
and differs from the discrete one, which seeks
transport couplings between empirical
measures~\citep{cuturi2013sinkhorn,chizat2018scaling}.

\section{Simultaneous Neural OT Method}
\label{sec:method}

In this section, we derive our novel optimization objective
for learning simultaneous UOT plans (\wasyparagraph \ref{sec:objective}) and propose an algorithm to solve (\wasyparagraph \ref{sec:training}).
We then study the approximation properties of shared
neural maps (\wasyparagraph \ref{sec:theoretical_properties}). The \textit{\underline{proofs of all theoretical results}} are given in Appendix \ref{app:theory}.

\subsection{Deriving the Optimization Objective}
\label{sec:objective}

Direct optimization of~\eqref{eq:sot_unbalanced} is difficult from samples because the marginal divergences involve unknown measures. We therefore derive an equivalent semi-dual objective whose dependence on the data distributions is only through expectations.

For the result below, assume that $\mathbb P_1,\dots,\mathbb P_K$ are atomless, $c:\mathcal X\times\mathcal Y\to[0,\infty)$ is finite and continuous, and that $\psi$ and $\phi$ satisfy the assumptions stated in Section~\ref{sec:background_ot} and $\phi(0)<+\infty$, $\psi(0)<+\infty$. Put
$
\overline{\mathbb P}=\frac1K\sum_{k=1}^K\mathbb P_k.
$
This is equivalent to atomlessness of $\overline{\mathbb P}$: an average of atomless measures is atomless, while each $\mathbb P_k$ is absolutely continuous with respect to $\overline{\mathbb P}$. The atomlessness assumption is used only in the purification argument of Lemma~\ref{lem:no-gap}. Let $J^\ast$ denote the infimum in~\eqref{eq:sot_unbalanced}.

\begin{theorem}[Semi-dual formulation of simultaneous UOT]
\label{thm:duality}
Under the assumptions above,
\begin{equation}
\label{eq:exact-duality}
J^\ast
=
\sup_{\mathbf v\in C_b(\mathcal Y)^K}
\inf_{\gamma}
\mathcal J(\mathbf v,\gamma)
=
\sup_{\mathbf v\in C_b(\mathcal Y)^K}
\inf_{T:\mathcal X\to\mathcal Y}
\mathcal J(\mathbf v,T),
\end{equation}
where the first infimum is over all probability kernels $\gamma(\cdot\mid x)$ from $\mathcal X$ to $\mathcal Y$ and
\begin{align}
\label{eq:dual-functional}
\mathcal J(\mathbf v,\gamma)
=
-\frac{1}{K}\sum_{k=1}^K
\Bigg[
&
\mathbb E_{x\sim\mathbb P_k}
\mathbb E_{y\sim\gamma(\cdot\mid x)}
\bar\psi\!\left(v_k(y)-c(x,y)\right)
+
\mathbb E_{y\sim\mathbb P^\ast}
\bar\phi\!\left(-v_k(y)\right)
\Bigg].
\end{align}
For a deterministic map $T$, we use the shorthand
$\mathcal J(\mathbf v,T):=\mathcal J(\mathbf v,\delta_{T(\cdot)})$.
\end{theorem}

In~\eqref{eq:dual-functional}, the
distributions enter only through expectations and can be estimated from samples.

\textbf{Remark on attainment.}
Theorem~\ref{thm:duality} concerns optimal values and does not require
attainment of the primal infimum. Existence is assumed explicitly whenever
it is needed. The deterministic representation in~\eqref{eq:exact-duality}
refers only to the inner semi-dual problem; the structure of primal
minimizers is addressed separately in Theorem~\ref{thm:monge}.

\subsection{Parametrization and Training}
\label{sec:training}

\textbf{Parametrization.}
To optimize over the shared conditional distribution
$\gamma(\cdot\mid x)$ in~\eqref{eq:exact-duality},
we represent it using a stochastic or deterministic
map, following~\citet{choi2024generative}.
Let $\mathcal Z\subseteq\mathbb R^{D_z}$ be a Borel set
and $\mathbb S\in\mathcal P(\mathcal Z)$
be an atomless noise distribution.
Every shared conditional distribution
$\gamma(\cdot\mid x)$ in~\eqref{eq:exact-duality}
admits a measurable realization
$T:\mathcal X\times\mathcal Z\to\mathcal Y$
with $z\sim\mathbb S$:
$
\gamma(\cdot\mid x)[=\gamma_T(\cdot\mid x)]=\left(T(x,\cdot)\right)_{\#}\mathbb S,$ where the map $T$ and the noise distribution $\mathbb S$
are shared across all sources.

Using this representation, we rewrite
the optimization problem~\eqref{eq:exact-duality} as
\begin{equation}
\label{eq:duality-map}
 J^\ast
=
\sup_{v_1,\dots,v_K\in C(\mathcal Y)}
\inf_T
\mathcal J(v_1,\dots,v_K,T),
\end{equation}
where the infimum is over measurable maps
$T:\mathcal X\times\mathcal Z\to\mathcal Y$ and
\vspace{-1mm}
\begin{equation}
\label{eq:map-functional}
\mathcal \!\!\!J(v_1,\dots,v_K,T)
\!=\!-\frac{1}{K}\sum_{k=1}^K
\Big[
\mathbb E_{\substack{x\sim\mathbb P_k\\z\sim\mathbb S}}
\bar\psi\!\left(v_k(T(x,z))\!-\! c(x,T(x,z))\right)
\!+\!
\mathbb E_{y\sim\mathbb P^\ast}
\bar\phi\!\left(-v_k(y)\right)
\Big].
\end{equation}

\vspace{-3mm}
In some settings, stochasticity of the conditional
distribution is not needed.
We can then consider a shared measurable map
$T:\mathcal X\to\mathcal Y$, which defines
deterministic conditional distributions
$\gamma_T(\cdot\mid x)=\delta_{T(x)}$.
In this case, the expectation over $\mathbb S$
in~\eqref{eq:map-functional} is omitted.
Theorem~\ref{thm:monge} below motivates this choice
when the transport cost is quadratic and both
marginal penalties are KL divergences.

To solve~\eqref{eq:duality-map}, we parametrize
the shared map and the potentials by neural nets
$T_\theta:\!\mathcal X\!\times\!\mathcal Z\to\!\mathcal\! Y$
and $v_{\omega_k}:\mathcal Y\to\mathbb R$,
$k=1,\dots,K$, with trainable parameters
$\theta$ and $\omega_1,\dots,\omega_K$, respectively.
The noise input is omitted for deterministic maps.
In practice, we scale the transport cost $c$ by an unbalancedness parameter
$\tau>0$: smaller values favor balanced simultaneous OT,
while larger values allow greater marginal deviations.

\textbf{Training.}
In our setup, the source distributions
$\mathbb P_1,\dots,\mathbb P_K$ and the target
distribution $\mathbb P^\ast$ are accessible only
through samples.
We therefore estimate the objective
$\mathcal J(v_{\omega_1},\dots,v_{\omega_K},T_\theta)$
in~\eqref{eq:map-functional} using Monte Carlo
with random mini-batches from the source, target
and noise distributions.

Training alternates between gradient ascent
with respect to the potential parameters
$\omega_1,\dots,\omega_K$ and gradient descent
with respect to the map parameters $\theta$.
Each potential is trained using its corresponding
source distribution and the common target,
while the shared map receives contributions
from all sources. The training procedure is summarized
in Algorithm~\ref{alg:simnot}.

\textbf{Inference.}
Given a new input $x$, we generate an output
$T_\theta(x,z)$ with $z\sim\mathbb S$,
or $T_\theta(x)$ in the deterministic case.
The source index and the potentials are not
required at inference time.

\begin{algorithm}[t!]
    \SetAlgorithmName{Algorithm}{empty}{Empty}
    \textbf{Input:}
    Distributions $\mathbb P_1,\dots,\mathbb P_K$,
    $\mathbb P^\ast$ and $\mathbb S$ accessible by samples;
    transport cost $c:\mathcal X\times\mathcal Y\to\mathbb R_+$;
    parameter $\tau>0$; conjugate functions
    $\bar\psi,\bar\phi$;
    shared map $T_\theta$ and potentials
    $v_{\omega_k}$, $k=1,\dots,K$;
    number $N_T$ of inner iterations; batch sizes.\\
    \textbf{Output:}
    Trained shared (stochastic) map $T_\theta$.

    \Repeat{converged}{
        Sample batches $X_k\sim\mathbb P_k$,
        $k=1,\dots,K$, and $Y\sim\mathbb P^\ast$\;
        For each $x\in X_k$, independently sample
        $z_{k,x}\sim\mathbb S$\;

        $\widehat{\mathcal L_v}
        \leftarrow
        -\frac{1}{K}\sum_{k=1}^K
        \Big[
        \frac{1}{|X_k|}\sum_{x\in X_k}
        \bar\psi\!\left(
        v_{\omega_k}(T_\theta(x,z_{k,x}))
        -\tau c(x,T_\theta(x,z_{k,x}))
        \right)
        +
        \frac{1}{|Y|}\sum_{y\in Y}
        \bar\phi\!\left(-v_{\omega_k}(y)\right)
        \Big]$\;

        Update $\omega_1,\dots,\omega_K$
        using $\nabla_{\omega_k}\widehat{\mathcal L_v}$
        to \textit{maximize} $\widehat{\mathcal L_v}$,
        keeping $\theta$ fixed\;

        \For{$n_T=1,2,\dots,N_T$}{
            Sample batches $X_k\sim\mathbb P_k$,
            $k=1,\dots,K$\;
            For each $x\in X_k$, independently sample
            $z_{k,x}\sim\mathbb S$\;

            $\widehat{\mathcal L_T}
            \leftarrow
            -\frac{1}{K}\sum_{k=1}^K
            \frac{1}{|X_k|}\sum_{x\in X_k}
            \bar\psi\!\left(
            v_{\omega_k}(T_\theta(x,z_{k,x}))
            -\tau c(x,T_\theta(x,z_{k,x}))
            \right)$\;

            Update $\theta$
            using $\nabla_\theta\widehat{\mathcal L_T}$
            to \textit{minimize} $\widehat{\mathcal L_T}$,
            keeping $\omega_1,\dots,\omega_K$ fixed\;
        }
    }
    \caption{Simultaneous Neural Optimal Transport (SimNOT)}
    \label{alg:simnot}
\end{algorithm}

\subsection{Theoretical Properties}
\label{sec:theoretical_properties}
The semi-dual formulation allows an arbitrary shared stochastic kernel. We first ask whether one neural map with a noise input is expressive enough to approximate such a kernel simultaneously for all source distributions.

For the approximation results in this subsection only, assume that
$
\mathcal Y=[a,b]^{D_y}
$
for some $a<b$, and let $\mathbb S=\operatorname{Unif}[0,1]$. The compact box assumption makes the $W_2$ statement automatic and allows coordinatewise clipping to $\mathcal Y$ to be implemented by ReLU operations.

\begin{theorem}[Approximation by a shared neural map]
\label{thm:neural-approximation}
Let $\gamma(\cdot\mid x)$ be a Borel probability kernel from $\mathcal X$ to $\mathcal Y$. For every $\varepsilon>0$, there exists a finite fully connected ReLU network
$$
T_\theta:\mathbb R^{D_x+1}\to\mathcal Y
$$
such that, with
$$
\gamma_\theta(\cdot\mid x)
=
\left(T_\theta(x,\cdot)\right)_\#\mathbb S,
$$
we have simultaneously for every $k=1,\dots,K$,
\begin{equation}
\label{eq:kernel-W2}
\int_{\mathcal X}
W_2^2\!\left(
\gamma_\theta(\cdot\mid x),
\gamma(\cdot\mid x)
\right)
\,d\mathbb P_k(x)
<\varepsilon
\end{equation}
and
\begin{equation}
\label{eq:output-W2}
W_2^2\!\left(
(T_\theta)_\#(\mathbb P_k\otimes\mathbb S),
\gamma_\#\mathbb P_k
\right)
<\varepsilon.
\end{equation}
\end{theorem}

The theorem shows that any shared stochastic transport
mechanism can be approximated by a single neural map, simultaneously and with a common accuracy bound for all source distributions.

\begin{corollary}[Consistency of the neural parametrization]
\label{cor:value-consistency}
Assume the conditions of Theorem~\ref{thm:duality} and $\mathcal Y=[a,b]^{D_y}$. Let $\mathcal T_n$ be increasing classes of ReLU maps $\mathcal X\times[0,1]\to\mathcal Y$ whose union contains every finite ReLU network with coordinatewise clipped output. Let $\mathcal V_m$ be increasing compact subsets of $C(\mathcal Y)^K$ whose union is uniformly dense in $C(\mathcal Y)^K$. Define
\begin{equation}
\label{eq:restricted_value}
D_{m,n}
=
\sup_{\mathbf v\in\mathcal V_m}
\inf_{T\in\mathcal T_n}
\mathcal J(\mathbf v,T).
\end{equation}
Then
\begin{equation}
\label{eq:value-consistency}
\lim_{m\to\infty}\lim_{n\to\infty}D_{m,n}
=
J^\ast.
\end{equation}
\end{corollary}

The corollary shows that the neural parametrization is also consistent
at the level of the optimization problem: as the map and potential
classes become sufficiently expressive, the restricted population value converges to the exact value $J^\ast$. This statement concerns optimal values; it does not require convergence of particular network parameters. The proof is given in Appendix~\ref{app:approximation}.

We finally ask a different question. The previous theorem shows that stochastic kernels can be represented by neural maps; it does not say whether randomization is actually needed by an optimal primal solution. For quadratic cost and KL marginal penalties, an attained optimum is deterministic under a standard regularity assumption on the source densities.

\begin{theorem}[Monge structure in the quadratic KL case]
\label{thm:monge}
Let $\mathcal X=\mathcal Y=\mathbb R^d$,
$$
c(x,y)=\|x-y\|^2,
\qquad
\psi(t)=\phi(t)=t\log t-t+1.
$$
Assume that there exists an open set
$\Omega\subseteq\mathbb R^d$ such that
$\mathbb P_k(\Omega)=1$ for all $k$, and the restriction of $\mathbb P_k$ to
$\Omega$ admits a strictly positive density $p_k \in C^2(\Omega)$
with respect to $d$-dimensional Lebesgue measure.
If the infimum in~\eqref{eq:sot_unbalanced}
is attained by $
\left(
\gamma^\ast,
\{(\gamma_k^\ast)_x\}_{k=1}^K
\right),
$
then there exists a Borel map
$T^\ast:\mathbb R^d\to\mathbb R^d$ such that
\begin{equation}
\label{eq:monge-conclusion}
\gamma^\ast(\cdot\mid x)
=
\delta_{T^\ast(x)}
\qquad
\sigma\text{-a.e.},
\qquad
\sigma=\sum_{k=1}^K(\gamma_k^\ast)_x.
\end{equation}
Consequently,
\[
\gamma_k^\ast
=
(\mathrm{id}_{\mathbb R^d},T^\ast)_{\#}
(\gamma_k^\ast)_x,
\qquad k=1,\dots,K.
\]
\end{theorem}

\section{Experiments}
\label{sec:experiments}

We evaluate SimNOT on a synthetic example and
an unpaired image restoration task.
In~\S\ref{sec:exp_toy}, we investigate the learned
shared transformation between several Gaussian source
distributions and a common Swiss roll target.
In~\S\ref{sec:exp_images}, we apply our method
to CelebA images with different degradations,
using a single model to restore them to a common
clean image distribution.

\subsection{Toy Experiment}
\label{sec:exp_toy}

We consider a two-dimensional example to illustrate
the ability of our method to learn a shared transformation
from several source distributions to a common target.
The sources $\mathbb P_1,\dots,\mathbb P_5$ are Gaussian
distributions with identity covariance matrices and
means $(0,0)$, $(3,0)$, $(0,3)$, $(-3,0)$, and $(0,-3)$,
respectively.
The target $\mathbb P^\ast$ is a noisy Swiss-roll
distribution.
We learn a single deterministic map
$T_\theta:\mathbb R^2\to\mathbb R^2$
with five source-specific potentials.
The map receives only the input coordinates,
without the source index.
Further experimental details are provided
in Appendix~\ref{app:swiss_roll}.

\textbf{Results.}
Figure~\ref{fig:swiss_roll} shows the transported samples
from each source alongside samples from the common target.
The shared map recovers the overall spiral structure
for all five sources.
The outputs differ in their spread around the spiral,
and some samples lie between its branches.
Thus, the experiment illustrates simultaneous
approximation of the target geometry, while leaving
visible discrepancies between the transported
and target distributions.

\begin{figure*}[t]
    \centering
    \includegraphics[height=0.19\textwidth]
{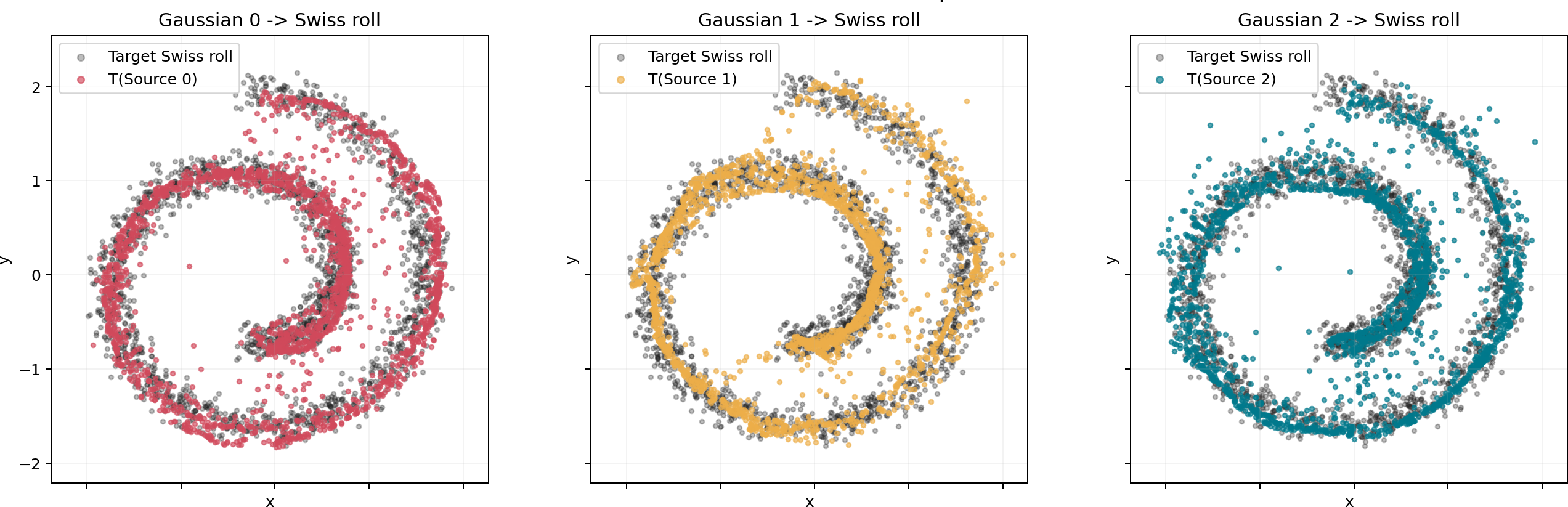}%
\hfill
\raisebox{-0.0115\textwidth}{%
    \includegraphics[height=0.203\textwidth]
    {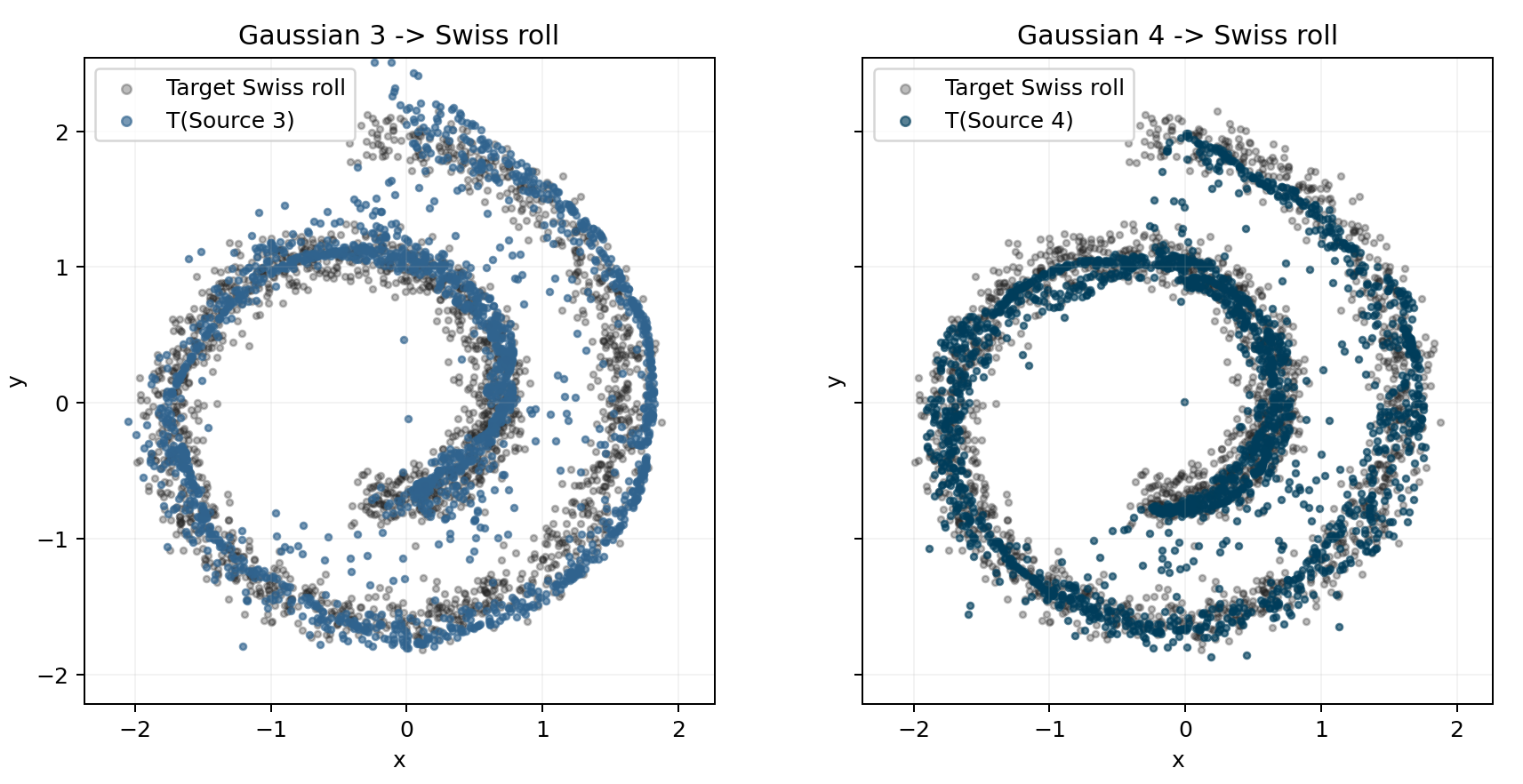}%
}
    \caption{
    Simultaneous transport from five Gaussian sources
    to a common Swiss roll target.
    Each panel shows the outputs of the same learned
    map for one source (colored points), overlaid
    with target samples (gray points).
    The visualization uses independent samples
    after 100K training iterations.
    }
    \label{fig:swiss_roll}
\end{figure*}

\subsection{Unpaired Image Restoration with Multiple Degradations}
\label{sec:exp_images}

\textbf{Experimental setup.}
We consider unpaired image restoration on CelebA
at resolution $64\times64$.
The five source distributions correspond to bicubic
and bilinear downsampling, JPEG compression,
Gaussian blur, and additive Gaussian noise.
The common target consists of clean images.
We use disjoint sets of images to construct the source
distributions, the clean target, and the test set.
Each source image is assigned to one degradation,
and its clean counterpart is not used as a training target.
Our goal is to learn a shared restoration map
that processes inputs without receiving their
degradation labels.
Data preparation and training details are provided
in Appendix~\ref{app:celeba_details}.

\textbf{Baselines.}
We compare SimNOT with two UOT baselines.
\textit{Pooled UOT} learns a shared map and a single
potential by treating the source distributions
as an equally weighted mixture.
\textit{Conditional UOT} conditions both the map
and the potential on the degradation label.
For blind restoration, this label is predicted
by a separately trained degradation classifier which uses the same five degradation
types and fixed parameters as restoration training.
All models use the same data split
and are evaluated using their final exponential
moving average (EMA) checkpoints after
$100$K training steps.

\textbf{Evaluation.}
We assess distribution matching using FID
and reconstruction accuracy using PSNR
against the corresponding clean images.
For the original degradation settings,
we compute FID separately for each source
and report the mean and maximum across
the five degradations.

\textbf{Results on training degradations.}
Table~\ref{tab:celeba_original} reports results on
held-out images with the degradation parameters
used during training.
SimNOT outperforms Pooled UOT in both mean
and maximum source-specific FID and in PSNR.
Conditional UOT achieves the best results
in this setting.
The degradation classifier attains $100\%$ accuracy,
so using its predictions yields the same results
as supplying the true degradation labels.

\begin{table}[t]
    \centering
    \small
    \setlength{\tabcolsep}{4pt}
    \begin{tabular}{lccc}
        \toprule
        Method
        & Mean FID $\downarrow$
        & Max FID $\downarrow$
        & PSNR $\uparrow$ \\
        \midrule
        Degraded input
        & 87.60 & 178.94 & 25.19 \\
        Pooled UOT
        & 9.96 & 14.39 & 27.51 \\
        Conditional UOT + classifier
        & \textbf{6.39} & \textbf{8.33}
        & \textbf{28.21} \\
        SimNOT (ours)
        & \underline{8.23} & \underline{11.88}
        & \underline{27.85} \\
        \bottomrule
    \end{tabular}
    \caption{
    CelebA restoration on held-out images with the training degradation parameters.
    Each of the five degradations is applied to the same full set
    of test images. Mean and Max FID aggregate the five source-specific FIDs.
    PSNR (dB) is averaged over images and degradation conditions.
    The best and second-best results among restoration
    methods are bold and underlined, respectively.
    }
    \label{tab:celeba_original}
\end{table}

\textbf{Generalization to unseen degradation strengths.}
We further evaluate the trained models on bilinear
downsampling with factors $\times3$ and $\times5$,
while training uses factor $\times4$.
All models, including the classifier, are applied
without retraining.
We additionally report LPIPS to assess perceptual
similarity to the clean reference images.

As shown in Table~\ref{tab:bilinear_generalization},
SimNOT achieves the lowest FID and LPIPS and
the highest PSNR among the restoration methods
at both factors.
The classifier fails to identify the bilinear
family in either setting.
To distinguish classification errors from
limitations of the conditional model,
we also evaluate Conditional UOT with the known
bilinear label.
It improves FID and LPIPS,
but SimNOT retains an advantage, particularly
at factor $\times5$.
Experiments with altered bicubic, JPEG, blur,
and noise parameters are reported in
Appendix~\ref{app:additional_results}.
While performance of SimNOT varies across degradations, it preserves perceptual similarity more consistently under shifts in degradation parameters.

\begin{table*}[t]
    \centering
    \small
    \setlength{\tabcolsep}{6pt}
    \begin{tabular}{lcccccc}
        \toprule
        & \multicolumn{3}{c}{Bilinear $\times3$}
        & \multicolumn{3}{c}{Bilinear $\times5$} \\
        \cmidrule(lr){2-4}
        \cmidrule(lr){5-7}
        Method
        & FID $\downarrow$
        & LPIPS $\downarrow$
        & PSNR $\uparrow$
        & FID $\downarrow$
        & LPIPS $\downarrow$
        & PSNR $\uparrow$ \\
        \midrule
        Degraded input
        & 92.82 & 0.1821 & 25.02
        & 317.44 & 0.3717 & 20.93 \\
        Pooled UOT
        & 27.71 & 0.0717 & 24.23
        & 99.22 & 0.2342 & 20.51 \\
        Conditional UOT + classifier
        & 35.16 & 0.1185 & 24.35
        & 80.87 & 0.2670 & 20.80 \\
        SimNOT (ours)
        & \textbf{18.44} & \textbf{0.0635} & \textbf{25.03}
        & \textbf{32.15} & \textbf{0.1134} & \textbf{20.93} \\
        \midrule
        Conditional UOT + known family
        & 21.42 & 0.0637 & 23.99
        & 77.23 & 0.1983 & 20.69 \\
        \bottomrule
    \end{tabular}
    \caption{
    Generalization to unseen bilinear downsampling factors.
    Models trained with factor $\times4$ are evaluated
    at factors $\times3$ and $\times5$ without retraining.
    Each setting uses all $20{,}259$ test images.
    The last row supplies the bilinear family label
    directly to Conditional UOT.
    Bold indicates the best result among restoration methods.
    }
    \label{tab:bilinear_generalization}
\end{table*}

\section{Conclusion}
\label{sec:conclusion}

We introduce SimNOT, a neural solver that brings simultaneous
OT to the continuous, sample-based setting.
Our method learns a single transport map for multiple source
distributions, with a separate distribution-matching objective
for each source.
We develop a divergence-based unbalanced formulation,
derive its exact semi-dual representation, and establish
neural approximation guarantees.
Experiments on synthetic data and image restoration
demonstrate that SimNOT learns shared transformations
from unpaired samples and handles multiple degradations
without requiring their labels at inference time.
Our work extends the scope of neural OT from pairwise
transport to simultaneous learning across multiple sources,
providing a principled framework for this broader class
of transformation problems. We discuss limitations and future research directions
in Appendix~\ref{app:limitations}.

\vspace{-3mm}
\subsection*{Reproducibility statement}

Experimental details, including data preparation,
architectures, hyperparameters, and evaluation protocols,
are provided in Appendix~\ref{app:experimental_details}.
Our code and instructions for reproducing the experiments
will be made publicly available.

\subsection*{AI use statement}

AI tools were used to assist with polishing and drafting the text, checking proofs, and supporting the implementation and analysis of experiments. All AI-assisted work was reviewed and verified by the authors, who take full responsibility for the final content of this work.

\bibliography{iclr2027_conference}
\bibliographystyle{iclr2027_conference}

\appendix

\section{Proofs of the theoretical results}
\label{app:theory}

\subsection{Proof of the semi-dual formulation}
\label{app:duality}

Put
$$
\overline{\mathbb P}
=
\frac{1}{K}\sum_{k=1}^K\mathbb P_k,
\qquad
r_k
=
\frac{d\mathbb P_k}{d\overline{\mathbb P}}.
$$
We choose Borel versions such that $0\leq r_k\leq K$ and $\sum_k r_k=K$ $\overline{\mathbb P}$-almost everywhere.

For $h=\psi$ or $h=\phi$, we use the standard variational identity
$$
D_h(\mu\|\nu)
=
\sup_{f\in C_b(E)}
\left\{
\int_E f\,d\mu
-
\int_E \bar h(f)\,d\nu
\right\}
$$
for finite nonnegative Borel measures on a Polish space $E$. Under the assumptions of Theorem~\ref{thm:duality}, these divergences are jointly lower semicontinuous for narrow convergence and satisfy the data-processing inequality.

For a probability kernel $\gamma(\cdot\mid x)$, define the probability measure
$$
m_k^\gamma(dx,dy)
=
\mathbb P_k(dx)\,\gamma(dy\mid x).
$$
Consider the convex relaxation
\begin{equation}
\label{eq:relaxed-problem}
R
=
\inf_{\gamma,\,\pi_1,\dots,\pi_K\geq0}
\frac{1}{K}\sum_{k=1}^K
\left[
\int c\,d\pi_k
+
D_\psi(\pi_k\|m_k^\gamma)
+
D_\phi((\pi_k)_y\|\mathbb P^\ast)
\right].
\end{equation}
Every admissible point of~\eqref{eq:sot_unbalanced} is admissible here: if
$$
\pi_k(dx,dy)
=
\alpha_k(dx)\,\gamma(dy\mid x),
$$
then
$$
D_\psi(\pi_k\|m_k^\gamma)
=
D_\psi(\alpha_k\|\mathbb P_k).
$$
Hence $R\leq J^\ast$.

The relaxation allows the density of $\pi_k$ with respect to $m_k^\gamma$ to depend on both $x$ and $y$. The following lemma shows that this additional freedom does not change the infimum.

\begin{lemma}[No relaxation gap]
\label{lem:no-gap}
Under the assumptions of Theorem~\ref{thm:duality},
$$
R=J^\ast.
$$
\end{lemma}

\begin{proof}
Fix an admissible point of~\eqref{eq:relaxed-problem} with finite value and write
$$
s_k
=
\frac{d\pi_k}{dm_k^\gamma}.
$$
We first reduce to bounded finite-valued densities supported on a common compact rectangle. By tightness, choose increasing compact rectangles $E_n=C_n\times D_n$ such that $m_k^\gamma(E_n^c)\to0$ for every $k$, and set
$$
s_{k,n}
=
2^{-n}\left\lfloor 2^n\min\{s_k,n\}\right\rfloor\mathbf 1_{E_n}.
$$
Then $0\leq s_{k,n}\uparrow s_k$ $m_k^\gamma$-almost everywhere. Monotone convergence gives convergence of the transport costs. Since for a nonnegative convex function $h$ and $0\leq a\leq b$,
$$
h(a)\leq h(0)+h(b),
$$
dominated convergence gives convergence of the source divergence terms. The target measures $(s_{k,n}m_k^\gamma)_y$ increase to $(\pi_k)_y$, and the same bound applied to their densities with respect to $\mathbb P^\ast$ gives convergence of the target divergences. It is therefore enough to treat finite-valued $s_k$ supported on one compact rectangle $C\times D$.

Fix $\delta>0$. Uniform continuity of $c$ on $C\times D$ gives a finite Borel partition
$$
D=B_1\sqcup\cdots\sqcup B_L
$$
such that
$$
\sup_{x\in C}\sup_{y,y'\in B_\ell}
|c(x,y)-c(x,y')|
\leq\delta
$$
for every $\ell$. Cells with $\mathbb P^\ast(B_\ell)=0$ carry no transported mass and may be ignored. For the remaining cells, put
$$
\nu_\ell
=
\frac{\mathbb P^\ast|_{B_\ell}}{\mathbb P^\ast(B_\ell)},
\qquad
\bar c_\ell(x)
=
\int_{B_\ell}c(x,y)\,d\nu_\ell(y).
$$

For fixed $x$, sample $y\sim\gamma(\cdot\mid x)$ and record the finite action
$$
(\ell,s_1(x,y),\dots,s_K(x,y))
$$
when $y\in B_\ell$; points carrying no transported mass are assigned one additional zero action. This gives a randomized rule with a finite action space. Since $\overline{\mathbb P}$ is atomless, it can be approximated by measurable deterministic rules while preserving, up to an arbitrarily small error, the finitely many integrals corresponding to
$$
r_k(x)\psi(z_k),
\qquad
r_k(x)z_k\mathbf 1_{\{\ell=j\}},
\qquad
r_k(x)z_k\bar c_\ell(x).
$$
One may obtain this directly by approximating these payoffs by simple functions on a common finite partition of $C$ and splitting each atomless partition cell according to the required action probabilities.

Let
$$
x\mapsto(\ell_n(x),z_{1,n}(x),\dots,z_{K,n}(x))
$$
be such deterministic rules. Define
$$
\alpha_{k,n}(dx)
=
z_{k,n}(x)\,\mathbb P_k(dx)
$$
on $C$ and set $\alpha_{k,n}=0$ outside $C$. At a nonzero action, define the common kernel by
$$
\gamma_n(dy\mid x)=\nu_{\ell_n(x)}(dy),
$$
and choose it arbitrarily at the zero action and outside $C$.

By construction, the source divergence terms converge to $D_\psi(\pi_k\|m_k^\gamma)$. The target measures have the form
$$
(\alpha_{k,n}(dx)\gamma_n(dy\mid x))_y
=
\sum_{\ell=1}^L b_{k\ell,n}\nu_\ell,
$$
where
$$
b_{k\ell,n}
\longrightarrow
b_{k\ell}:=(\pi_k)_y(B_\ell).
$$
If $\rho_k=d(\pi_k)_y/d\mathbb P^\ast$, then Jensen's inequality gives
$$
\mathbb P^\ast(B_\ell)
\phi\!\left(
\frac{b_{k\ell}}{\mathbb P^\ast(B_\ell)}
\right)
\leq
\int_{B_\ell}\phi(\rho_k)\,d\mathbb P^\ast.
$$
Hence the limiting target divergence is no larger than $D_\phi((\pi_k)_y\|\mathbb P^\ast)$.

Finally, the limiting transport cost differs from $\int c\,d\pi_k$ by at most
$$
\delta\,\pi_k(\mathcal X\times\mathcal Y).
$$
Letting the purification error tend to zero, then $\delta\downarrow0$, and finally removing the initial finite-valued compact approximation gives $J^\ast\leq R$. Together with $R\leq J^\ast$, this proves the lemma.
\end{proof}

\begin{proof}[Proof of Theorem~\ref{thm:duality}]
By Lemma~\ref{lem:no-gap}, it remains to derive the dual of the convex problem~\eqref{eq:relaxed-problem}.

We first explain the argument for a compact target. Write a shared kernel as
$$
q(dx,dy)
=
\overline{\mathbb P}(dx)\,\gamma(dy\mid x)
$$
and set
$$
Q
=
\left\{
q\in\mathcal P(\mathcal X\times\mathcal Y):
q_x=\overline{\mathbb P}
\right\}.
$$
Then
$$
m_k^\gamma=r_k q.
$$
The set $Q$ is convex and narrowly compact. Moreover, $q\mapsto r_kq$ is narrowly continuous on $Q$: approximate the bounded Borel function $r_k$ in $L^1(\overline{\mathbb P})$ by bounded continuous functions and use that every $q\in Q$ has the same $x$-marginal.

For finite nonnegative measures $\pi_k$ on $\mathcal X\times\mathcal Y$ and $\beta_k$ on $\mathcal Y$, define
$$
E(q,\boldsymbol\pi,\boldsymbol\beta)
=
\sum_{k=1}^K
\left[
\int c\,d\pi_k
+
D_\psi(\pi_k\|r_kq)
+
D_\phi(\beta_k\|\mathbb P^\ast)
\right].
$$
Its sublevel sets are narrowly compact. Indeed, superlinearity of $\psi$ gives uniform bounds on the masses and tails of $\pi_k$ relative to the uniformly tight family $\{r_kq:q\in Q\}$; the same argument applies to $\beta_k$ relative to $\mathbb P^\ast$. Lower semicontinuity follows from lower semicontinuity of the cost and divergences and from continuity of $q\mapsto r_kq$.

For signed finite measures $\eta_1,\dots,\eta_K$ on $\mathcal Y$, define
$$
W(\boldsymbol\eta)
=
\inf_{\substack{q\in Q,\;\pi_k,\beta_k\geq0\\(\pi_k)_y-\beta_k=\eta_k}}
E(q,\boldsymbol\pi,\boldsymbol\beta).
$$
The function $W$ is proper, convex and lower semicontinuous in the weak topology induced by $C(\mathcal Y)^K$. Its finite infima are attained by compactness of the energy sublevels, and
$$
W(0)=KR.
$$
Fenchel--Moreau therefore gives
$$
KR
=
\sup_{\mathbf v\in C(\mathcal Y)^K}
\{-W^\ast(\mathbf v)\}.
$$

We compute the conjugate pointwise. For fixed $q$ and $v\in C(\mathcal Y)$,
$$
\sup_{\pi\geq0}
\left\{
\int v(y)\,d\pi
-
\int c\,d\pi
-
D_\psi(\pi\|r_kq)
\right\}
=
\int
\bar\psi(v(y)-c(x,y))
\,d(r_kq)(x,y).
$$
To justify the interchange of supremum and integration, note that $v-c$ is bounded above. Superlinearity of $\psi$ therefore bounds all pointwise maximizers in a common compact interval. Using a countable dense subset of the effective domain (including its finite endpoints), one obtains measurable finite-valued near-maximizers. Choosing them at least as good as $t=0$ also gives a uniform bound on $c(x,y)t$, so the corresponding measures have finite transport cost. The resulting integrals converge to the pointwise conjugate. Similarly,
$$
\sup_{\beta\geq0}
\left\{
-\int v\,d\beta
-
D_\phi(\beta\|\mathbb P^\ast)
\right\}
=
\int\bar\phi(-v)\,d\mathbb P^\ast.
$$
Substitution into the Fenchel--Moreau formula yields
$$
R
=
\sup_{\mathbf v\in C(\mathcal Y)^K}
\inf_\gamma
\mathcal J(\mathbf v,\gamma)
$$
when $\mathcal Y$ is compact.

Now suppose that $\mathcal Y$ is noncompact. Since it is a closed subset of Euclidean space, its one-point compactification
$$
\widehat{\mathcal Y}=\mathcal Y\cup\{\infty\}
$$
is a compact metric space. Extend $\mathbb P^\ast$ by zero at $\infty$ and set
$$
\widehat c(x,y)=c(x,y)
\quad\text{for }y\in\mathcal Y,
\qquad
\widehat c(x,\infty)=0.
$$
The extended cost is lower semicontinuous because $c\geq0$. Compactification does not change the relaxed value. Indeed, every finite-value candidate has $\pi_k(\mathcal X\times\{\infty\})=0$ because $(\pi_k)_y\ll\mathbb P^\ast$. If the reference kernel places mass at $\infty$, send this mass to an arbitrary fixed point $y_0\in\mathcal Y$. The transport plans and their target marginals are unchanged, while the data-processing inequality shows that the source divergence cannot increase.

Applying the compact argument on $\widehat{\mathcal Y}$ and then restricting the potentials and kernels back to $\mathcal Y$ gives
$$
R
\leq
\sup_{\mathbf v\in C_b(\mathcal Y)^K}
\inf_\gamma\mathcal J(\mathbf v,\gamma).
$$
Indeed, restrictions of functions in $C(\widehat{\mathcal Y})$ belong to $C_b(\mathcal Y)$, while kernels supported on $\mathcal Y$ form a subclass of the compactified kernels. For the reverse bound, let $v_k\in C_b(\mathcal Y)$ and use the pointwise Fenchel inequalities
$$
\int c\,d\pi_k
+
D_\psi(\pi_k\|m_k^\gamma)
\geq
\int v_k\,d(\pi_k)_y
-
\int\bar\psi(v_k-c)\,dm_k^\gamma
$$
and
$$
D_\phi((\pi_k)_y\|\mathbb P^\ast)
\geq
-\int v_k\,d(\pi_k)_y
-
\int\bar\phi(-v_k)\,d\mathbb P^\ast.
$$
After summing and taking infima, this gives
$$
\sup_{\mathbf v\in C_b(\mathcal Y)^K}
\inf_\gamma\mathcal J(\mathbf v,\gamma)
\leq R.
$$
Together with the compactified lower bound and Lemma~\ref{lem:no-gap},
$$
J^\ast
=
\sup_{\mathbf v\in C_b(\mathcal Y)^K}
\inf_\gamma\mathcal J(\mathbf v,\gamma).
$$

It remains to remove the inner randomization. Fix $\mathbf v\in C_b(\mathcal Y)^K$ and define
$$
H_{\mathbf v}(x,y)
=
\frac{1}{K}\sum_{k=1}^K
r_k(x)\bar\psi(v_k(y)-c(x,y)).
$$
For each fixed $x$, the function $y\mapsto H_{\mathbf v}(x,y)$ is continuous, and $H_{\mathbf v}$ is bounded because
$$
-\psi(0)
\leq
\bar\psi(v_k(y)-c(x,y))
\leq
\bar\psi(\|v_k\|_\infty).
$$
Choose a countable dense set $\{y_j:j\geq1\}\subseteq\mathcal Y$. Then
$$
h(x)
=
\sup_{y\in\mathcal Y}H_{\mathbf v}(x,y)
=
\sup_{j\geq1}H_{\mathbf v}(x,y_j)
$$
is Borel. For every $\varepsilon>0$, choosing the first $y_j$ with
$$
H_{\mathbf v}(x,y_j)>h(x)-\varepsilon
$$
defines a Borel map $T_\varepsilon:\mathcal X\to\mathcal Y$. Hence
$$
\sup_\gamma
\int H_{\mathbf v}(x,y)\,\gamma(dy\mid x)\,d\overline{\mathbb P}(x)
=
\sup_T
\int H_{\mathbf v}(x,T(x))\,d\overline{\mathbb P}(x),
$$
where the second supremum is over Borel maps. Since the target term in $\mathcal J$ does not depend on the kernel,
$$
\inf_\gamma\mathcal J(\mathbf v,\gamma)
=
\inf_T\mathcal J(\mathbf v,T).
$$
This proves Theorem~\ref{thm:duality}.
\end{proof}

\subsection{Proof of the approximation results}
\label{app:approximation}

\begin{proof}[Proof of Theorem~\ref{thm:neural-approximation}]
By the randomization lemma, there exists a measurable map
$$
T:\mathcal X\times[0,1]\to\mathcal Y
$$
such that
$$
(T(x,\cdot))_\#\mathbb S
=
\gamma(\cdot\mid x).
$$
Put
$$
\overline{\mathbb P}
=
\frac{1}{K}\sum_{k=1}^K\mathbb P_k,
\qquad
\mu
=
\overline{\mathbb P}\otimes\mathbb S.
$$
Since $\mathcal Y=[a,b]^{D_y}$ is bounded, $T\in L^2(\mu;\mathbb R^{D_y})$.

We claim that $T$ can be approximated arbitrarily well in $L^2(\mu)$ by finite ReLU networks with output in $\mathcal Y$. By Lusin's theorem, for every $\eta>0$ there is a compact set $K_\eta\subseteq\mathcal X\times[0,1]$ such that $T|_{K_\eta}$ is continuous and $\mu(K_\eta^c)<\eta$. The usual ReLU universal approximation theorem gives a finite network that approximates $T$ uniformly on $K_\eta$. Clipping each output coordinate to $[a,b]$ is itself a ReLU operation and cannot increase the error to a point of $\mathcal Y$. Since both $T$ and the clipped network are uniformly bounded, the contribution of $K_\eta^c$ to the $L^2$ error tends to zero with $\eta$. Thus, for any $\delta>0$, we can choose $T_\theta$ such that
$$
\int
\|T_\theta(x,z)-T(x,z)\|^2
\,d\overline{\mathbb P}(x)\,d\mathbb S(z)
<\delta.
$$

Because $\mathbb P_k\leq K\overline{\mathbb P}$,
$$
\int
\|T_\theta(x,z)-T(x,z)\|^2
\,d\mathbb P_k(x)\,d\mathbb S(z)
< K\delta
$$
for every $k$. Choose $\delta<\varepsilon/K$.

For fixed $x$, the pair
$$
(T_\theta(x,Z),T(x,Z)),
\qquad Z\sim\mathbb S,
$$
is a coupling of $\gamma_\theta(\cdot\mid x)$ and $\gamma(\cdot\mid x)$. Therefore
$$
W_2^2\!\left(
\gamma_\theta(\cdot\mid x),
\gamma(\cdot\mid x)
\right)
\leq
\mathbb E
\|T_\theta(x,Z)-T(x,Z)\|^2.
$$
Integrating with respect to $\mathbb P_k$ proves~\eqref{eq:kernel-W2}. Sampling additionally $X\sim\mathbb P_k$ gives a coupling of the two output laws with the same quadratic cost, which proves~\eqref{eq:output-W2}.
\end{proof}

\begin{proof}[Proof of Corollary~\ref{cor:value-consistency}]
For fixed $\mathbf v\in C(\mathcal Y)^K$, set
$$
d_n(\mathbf v)
=
\inf_{T\in\mathcal T_n}\mathcal J(\mathbf v,T),
\qquad
d(\mathbf v)
=
\inf_T\mathcal J(\mathbf v,T),
$$
where the second infimum is over all measurable maps $T:\mathcal X\times[0,1]\to\mathcal Y$.

The proof of Theorem~\ref{thm:neural-approximation} shows that every such measurable map can be approximated in $L^2(\overline{\mathbb P}\otimes\mathbb S)$ by maps from $\bigcup_n\mathcal T_n$. Hence the approximating maps converge in probability under every $\mathbb P_k\otimes\mathbb S$. Since $c$ and the potentials are continuous and
$$
-\psi(0)
\leq
\bar\psi(v_k(y)-c(x,y))
\leq
\bar\psi(\|v_k\|_\infty),
$$
the corresponding integrands are uniformly bounded. It follows that the expectations converge, and therefore
$$
d_n(\mathbf v)\downarrow d(\mathbf v).
$$

Fix $m$. Compactness of $\mathcal V_m$ in the uniform norm gives a common bound on its potentials. Superlinearity of $\psi$ implies that $\bar\psi$ is Lipschitz on every interval $(-\infty,B]$, while $\bar\phi$ is Lipschitz on bounded intervals. Consequently, $d_n$ and $d$ are uniformly Lipschitz on $\mathcal V_m$. Since $d_n\downarrow d$ and $\mathcal V_m$ is compact, Dini's theorem gives uniform convergence and therefore
$$
\lim_{n\to\infty}D_{m,n}
=
\sup_{\mathbf v\in\mathcal V_m}d(\mathbf v).
$$

Finally, $\bigcup_m\mathcal V_m$ is uniformly dense in $C(\mathcal Y)^K$ and $d$ is continuous. Hence
$$
\lim_{m\to\infty}
\sup_{\mathbf v\in\mathcal V_m}d(\mathbf v)
=
\sup_{\mathbf v\in C(\mathcal Y)^K}d(\mathbf v).
$$
Because $\mathcal Y$ is compact, $C(\mathcal Y)=C_b(\mathcal Y)$, and Theorem~\ref{thm:duality} identifies the last quantity with $J^\ast$. This proves~\eqref{eq:value-consistency}.
\end{proof}

\subsection{Proof of the deterministic structure theorem}
\label{app:monge}

\begin{proof}[Proof of Theorem~\ref{thm:monge}]
Write
$$
\alpha_k=\alpha_k^\ast,
\qquad
M=\sum_{k=1}^K\mathbb P_k,
\qquad
\sigma=\sum_{k=1}^K\alpha_k.
$$
Finite KL divergence gives
$$
\sigma\ll M\ll\mathcal L^d.
$$
If $\sigma=0$, the conclusion is vacuous. Assume $\sigma\neq0$ and define on $\Omega$
$$
a_k(x)
=
\frac{p_k(x)}{\sum_jp_j(x)},
\qquad
\lambda_k(x)
=
\frac{d\alpha_k}{d\sigma}(x).
$$
Then $a_k>0$, $\sum_k a_k=1$, and $\sum_k\lambda_k=1$ $\sigma$-almost everywhere. The functions $a_k$ are $C^2$ on $\Omega$.

Let $m=d\sigma/dM$. Since
$$
\frac{d\alpha_k}{d\mathbb P_k}
=
\frac{m\lambda_k}{a_k},
$$
a direct calculation gives
\begin{equation}
\label{eq:KL-chain}
\sum_{k=1}^K
\operatorname{KL}(\alpha_k\|\mathbb P_k)
=
\operatorname{KL}(\sigma\|M)
+
\int_\Omega
\sum_{k=1}^K
\lambda_k(x)\log\frac{\lambda_k(x)}{a_k(x)}
\,d\sigma(x).
\end{equation}

Let
$$
\Delta_K
=
\left\{
z\in[0,1]^K:\sum_{k=1}^Kz_k=1
\right\}
$$
be the probability simplex, and write $\lambda(x)=(\lambda_1(x),\dots,\lambda_K(x))$. Define the finite measure
$$
\zeta(dx,dy,dz)
=
\sigma(dx)\,\gamma^\ast(dy\mid x)\,\delta_{\lambda(x)}(dz)
$$
on $\Omega\times\mathbb R^d\times\Delta_K$, and let $\xi$ be its $(y,z)$-marginal. Consider the cost
$$
\widetilde c(x;y,z)
=
\|x-y\|^2
+
\sum_{k=1}^K
z_k\log\frac{z_k}{a_k(x)},
\qquad
0\log0:=0.
$$
We first show that $\zeta$ is an optimal transport plan between $\sigma$ and $\xi$ for this cost.

Let $\widetilde\zeta\in\Pi(\sigma,\xi)$ and let $\eta$ be its $(x,y)$-marginal. Write
$$
\eta(dx,dy)
=
\sigma(dx)\,\widetilde\gamma(dy\mid x)
$$
and define
$$
\widetilde\pi_k(dx,dy)
=
\int_{\Delta_K}z_k\,\widetilde\zeta(dx,dy,dz).
$$
If
$$
\overline z_k(x,y)
=
\mathbb E_{\widetilde\zeta}[z_k\mid x,y],
$$
then
$$
\widetilde\pi_k(dx,dy)
=
\overline z_k(x,y)\,\eta(dx,dy),
\qquad
\sum_k\overline z_k=1.
$$
Because the $(y,z)$-marginal of $\widetilde\zeta$ is fixed,
$$
(\widetilde\pi_k)_y(B)
=
\int_{B\times\Delta_K}z_k\,d\xi(y,z),
$$
so these target marginals are the same as those of the original optimal plans.

Using the same density calculation as in~\eqref{eq:KL-chain}, now on $\mathcal X\times\mathbb R^d$, we obtain
$$
\begin{aligned}
&\sum_{k=1}^K
\operatorname{KL}\!\left(
\widetilde\pi_k
\middle\|
\mathbb P_k(dx)\,\widetilde\gamma(dy\mid x)
\right)
\\
&\qquad=
\operatorname{KL}(\sigma\|M)
+
\int
\sum_{k=1}^K
\overline z_k(x,y)
\log\frac{\overline z_k(x,y)}{a_k(x)}
\,d\eta(x,y)
\\
&\qquad\leq
\operatorname{KL}(\sigma\|M)
+
\int
\sum_{k=1}^K
z_k\log\frac{z_k}{a_k(x)}
\,d\widetilde\zeta(x,y,z),
\end{aligned}
$$
where the inequality is conditional Jensen. Moreover,
$$
\sum_{k=1}^K
\int\|x-y\|^2\,d\widetilde\pi_k(x,y)
=
\int\|x-y\|^2\,d\widetilde\zeta(x,y,z).
$$
Thus, if $\widetilde\zeta$ had strictly smaller $\widetilde c$-cost than $\zeta$, the measures $\widetilde\pi_k$ together with the common kernel $\widetilde\gamma$ would give an admissible point of the relaxed problem~\eqref{eq:relaxed-problem} with value strictly below $J^\ast$. This contradicts Lemma~\ref{lem:no-gap}. Hence $\zeta$ is optimal.

It remains to show that its $y$-coordinate is uniquely determined by $x$. Cover $\Omega$ by countably many open balls $B_j$ with $\overline B_j\subset\Omega$, and choose larger open balls $U_j$ such that
$$
\overline B_j\subset U_j,
\qquad
\overline U_j\subset\Omega.
$$
For $\ell\geq1$, let
$$
D_\ell
=
\{y\in\mathbb R^d:\|y\|\leq\ell\}.
$$
The restriction of $\zeta$ to
$$
\overline B_j\times D_\ell\times\Delta_K
$$
is optimal between its own marginals: otherwise it could be replaced by a cheaper coupling without changing the full marginals of $\zeta$.

On $\overline U_j\times D_\ell\times\Delta_K$, the cost $\widetilde c$ is continuous and continuously differentiable in $x$, with bounded $x$-gradient. Compact Kantorovich duality gives dual potentials. Taking the cost transform on the source side gives a source potential $u_{j\ell}$ that is Lipschitz on $U_j$, hence differentiable Lebesgue-almost everywhere.

At a differentiability point where complementary slackness holds,
$$
\nabla u_{j\ell}(x)
=
2(x-y)
-
\sum_{k=1}^Kz_k\nabla\log a_k(x).
$$
For $\zeta$ we have $z=\lambda(x)$ almost everywhere, so the conditional measure can charge only
$$
y
=
x
-
\frac{1}{2}
\left[
\nabla u_{j\ell}(x)
+
\sum_{k=1}^K
\lambda_k(x)\nabla\log a_k(x)
\right].
$$
Since $\sigma\ll\mathcal L^d$, after removing one $\sigma$-null set for the countable family $(j,\ell)$, this conclusion holds whenever the corresponding restricted conditional measure is nonzero. The sets $D_\ell$ increase to $\mathbb R^d$, and their restrictions are consistent. Therefore $\gamma^\ast(\cdot\mid x)$ is a Dirac mass for $\sigma$-almost every $x$.

A Borel probability kernel that is almost everywhere a Dirac mass induced, after modification on a null set, by a Borel map. Hence there exists $T^\ast:\mathbb R^d\to\mathbb R^d$ satisfying~\eqref{eq:monge-conclusion}, and the formula for the optimal plans follows immediately.
\end{proof}

\section{Extended Discussion of Related Work}
\label{app:related_work}

\textbf{OT barycenters.}
Barycenter methods also relate several distributions through OT.
Continuous neural solvers estimate a common distribution by
minimizing an average of transport
costs~\citep{fan2021scalable, kolesov2024energyguided, kolesovestimating}; semi-unbalanced extensions allow
robust barycenter estimation~\citep{gazdieva2025robust}.
In these problems, the common distribution is an unknown to
be optimized.
In our setting, the target distribution is prescribed and
accessible through samples.
Moreover, the simultaneous constraint requires a shared
transformation, whereas the standard barycenter problem allows
separate transport plans for different input distributions.

\textbf{All-in-one image restoration.}
Learning a single model for multiple degradations is the aim
of all-in-one image restoration.
AirNet~\citep{li2022all} learns degradation representations,
while PromptIR~\citep{potlapalli2023promptir} uses learned prompts
to adapt restoration to the input.
Their restoration networks are trained with paired degraded
and clean images.
OT-based approaches also address this task:
DA-RCOT~\citep{tang2025degradation} considers both unpaired and
paired restoration, using transport residuals to inform the
cost and condition the restoration map.
Our focus is on imposing separate distribution-matching
objectives for the individual sources within a simultaneous
formulation.
Another closely related approach,
BaryIR~\citep{tang2026learning}, learns a shared map and multiple
potentials to construct a Wasserstein barycenter representation
of degraded features.
It combines this representation with residual features and
uses paired supervision for restoration.
SimNOT instead learns transport to a prescribed target
distribution from unpaired samples.
Thus, the key distinction from BaryIR lies in the transport
problem and supervision, rather than in sharing a map across
degradations.

\section{Additional Experimental Results}
\label{app:additional_results}

\subsection{Results for Individual Degradations}
\label{app:per_degradation_results}

Table~\ref{tab:celeba_per_degradation} provides
the source-specific results underlying
Table~\ref{tab:celeba_original}.
It shows that SimNOT achieves lower FID and higher PSNR than Pooled UOT
for JPEG compression, Gaussian blur, and Gaussian noise,
although Pooled UOT performs slightly better for the
downsampling degradations.
Besides, SimNOT obtains the best FID and PSNR for Gaussian noise
and the best PSNR for JPEG compression, while conditional UOT performs best in FID for the remaining
four degradations.

\begin{table*}[h]
    \centering
    \small
    \setlength{\tabcolsep}{6pt}
    \begin{tabular}{lccccc}
        \toprule
        Method
        & Bicubic & Bilinear & JPEG & Blur & Noise \\
        \midrule
        \multicolumn{6}{c}{FID $\downarrow$} \\
        \midrule
        Degraded input
        & 124.86 & 178.94 & 31.19 & 46.47 & 56.54 \\
        Pooled UOT
        & \underline{10.96} & \underline{8.54}
        & 9.15 & 6.75 & 14.39 \\
        Conditional UOT + classifier
        & \textbf{8.33} & \textbf{8.18}
        & \textbf{5.92} & \textbf{3.48}
        & \underline{6.02} \\
        SimNOT (ours)
        & 11.28 & 11.88
        & \underline{6.59} & \underline{6.27}
        & \textbf{5.14} \\
        \midrule
        \multicolumn{6}{c}{PSNR (dB) $\uparrow$} \\
        \midrule
        Degraded input
        & 22.89 & 23.02 & 28.50 & 25.22 & 26.32 \\
        Pooled UOT
        & \underline{22.96} & \textbf{23.74}
        & 29.22 & 29.84 & 31.78 \\
        Conditional UOT + classifier
        & \textbf{22.99} & \underline{23.72}
        & \underline{29.55} & \textbf{31.34}
        & \underline{33.42} \\
        SimNOT (ours)
        & 22.70 & 23.26
        & \textbf{29.59} & \underline{29.86}
        & \textbf{33.85} \\
        \bottomrule
    \end{tabular}
    \caption{
    Source-specific results on the original
    degradation settings.
    The \underline{\textit{evaluation protocol}} is described
    in Appendix~\ref{app:celeba_details}.
    The best and second-best results among restoration
    methods are bold and underlined, respectively.
    Rankings use unrounded values.
    }
    \label{tab:celeba_per_degradation}
\end{table*}

\subsection{Generalization to Unseen Degradation Strengths}
\label{app:degradation_shifts}

We extend the bilinear experiments
in Table~\ref{tab:bilinear_generalization}
to altered bicubic downsampling factors,
JPEG quality levels, blur standard deviations,
and noise standard deviations.
Tables~\ref{tab:shifts_milder}
and~\ref{tab:shifts_stronger} report results
for milder and stronger degradations, respectively.
Each setting uses all $20{,}259$ test images
and the same trained models, without fine-tuning.

The results show that SimNOT achieves lower LPIPS than classifier-routed Conditional
UOT in $9/10$ conditions and Pooled UOT in $7/10$.
SimNOT also has the highest mean PSNR among the blind methods,
although Pooled UOT achieves a slightly lower mean FID.
The known-family reference has better mean scores in all three metrics.
Thus, among the blind methods, SimNOT's most consistent advantage is in
perceptual similarity measured by LPIPS.

\begin{table*}[t]
    \centering
    \small
    \setlength{\tabcolsep}{10pt}
    \begin{tabular}{lccc}
        \toprule
        Method
        & FID $\downarrow$
        & LPIPS $\downarrow$
        & PSNR $\uparrow$ \\
        \midrule
        \multicolumn{4}{c}{Bicubic $\times3$
        (training: $\times4$)} \\
        \midrule
        Degraded input
        & 61.46 & 0.16835 & 25.24 \\
        Pooled UOT
        & \underline{21.06} & \underline{0.07670}
        & \textbf{23.10} \\
        Conditional UOT + classifier
        & 113.85 & 0.20161 & 20.25 \\
        SimNOT (ours)
        & \textbf{18.72} & \textbf{0.06642}
        & \underline{22.49} \\
        \cmidrule(lr){1-4}
        Conditional UOT + known family
        & 9.27 & 0.04598 & 23.74 \\
        \midrule
        \multicolumn{4}{c}{JPEG quality $30$
        (training: $25$)} \\
        \midrule
        Degraded input
        & 27.22 & 0.03445 & 29.11 \\
        Pooled UOT
        & 8.30 & 0.01726 & 29.74 \\
        Conditional UOT + classifier
        & \textbf{5.48} & \textbf{0.01546}
        & \underline{30.00} \\
        SimNOT (ours)
        & \underline{6.40} & \underline{0.01598}
        & \textbf{30.02} \\
        \cmidrule(lr){1-4}
        Conditional UOT + known family
        & 5.48 & 0.01546 & 30.00 \\
        \midrule
        \multicolumn{4}{c}{Gaussian blur $\sigma=1.25$
        (training: $1.5$)} \\
        \midrule
        Degraded input
        & 31.02 & 0.19583 & 26.62 \\
        Pooled UOT
        & \underline{6.80} & \textbf{0.01804}
        & \textbf{30.26} \\
        Conditional UOT + classifier
        & \textbf{6.61} & 0.03509
        & \underline{29.09} \\
        SimNOT (ours)
        & 8.95 & \underline{0.03057} & 28.41 \\
        \cmidrule(lr){1-4}
        Conditional UOT + known family
        & 6.61 & 0.03508 & 29.09 \\
        \midrule
        \multicolumn{4}{c}{Gaussian noise $\sigma=0.04$
        (training: $0.05$)} \\
        \midrule
        Degraded input
        & 48.45 & 0.02418 & 28.20 \\
        Pooled UOT
        & 6.60 & 0.00867 & 32.85 \\
        Conditional UOT + classifier
        & \textbf{2.85} & \underline{0.00543}
        & \underline{34.60} \\
        SimNOT (ours)
        & \underline{3.17} & \textbf{0.00510}
        & \textbf{35.08} \\
        \cmidrule(lr){1-4}
        Conditional UOT + known family
        & 2.85 & 0.00543 & 34.60 \\
        \bottomrule
    \end{tabular}
    \caption{
    Generalization to milder degradations.
    All models are evaluated without retraining
    on $20{,}259$ images per setting.
    The best and second-best results among blind
    restoration methods are bold and underlined,
    respectively.
    The known-family variant receives the true
    degradation family and is reported separately
    as a reference.
    Rankings use unrounded values.
    }
    \label{tab:shifts_milder}
\end{table*}

\begin{table*}[t]
    \centering
    \small
    \setlength{\tabcolsep}{10pt}
    \begin{tabular}{lccc}
        \toprule
        Method
        & FID $\downarrow$
        & LPIPS $\downarrow$
        & PSNR $\uparrow$ \\
        \midrule
        \multicolumn{4}{c}{Bicubic $\times5$
        (training: $\times4$)} \\
        \midrule
        Degraded input
        & 238.78 & 0.41126 & 20.67 \\
        Pooled UOT
        & \textbf{40.31} & \textbf{0.16003}
        & \textbf{20.09} \\
        Conditional UOT + classifier
        & 177.46 & 0.27675 & \underline{19.69} \\
        SimNOT (ours)
        & \underline{151.24} & \underline{0.24771}
        & 19.49 \\
        \cmidrule(lr){1-4}
        Conditional UOT + known family
        & 21.89 & 0.09326 & 20.37 \\
        \midrule
        \multicolumn{4}{c}{JPEG quality $20$
        (training: $25$)} \\
        \midrule
        Degraded input
        & 36.78 & 0.05387 & 27.73 \\
        Pooled UOT
        & 10.94 & 0.02463 & 28.52 \\
        Conditional UOT + classifier
        & \textbf{7.16} & \underline{0.02094}
        & \underline{28.89} \\
        SimNOT (ours)
        & \underline{7.30} & \textbf{0.02092}
        & \textbf{28.95} \\
        \cmidrule(lr){1-4}
        Conditional UOT + known family
        & 7.16 & 0.02094 & 28.89 \\
        \midrule
        \multicolumn{4}{c}{Gaussian blur $\sigma=1.75$
        (training: $1.5$)} \\
        \midrule
        Degraded input
        & 64.50 & 0.30943 & 24.13 \\
        Pooled UOT
        & \underline{10.91} & \textbf{0.03433}
        & 27.85 \\
        Conditional UOT + classifier
        & \textbf{9.88} & 0.04010 & \textbf{28.86} \\
        SimNOT (ours)
        & 11.82 & \underline{0.03488}
        & \underline{28.38} \\
        \cmidrule(lr){1-4}
        Conditional UOT + known family
        & 9.88 & 0.04010 & 28.86 \\
        \midrule
        \multicolumn{4}{c}{Gaussian noise $\sigma=0.06$
        (training: $0.05$)} \\
        \midrule
        Degraded input
        & 65.16 & 0.06515 & 24.79 \\
        Pooled UOT
        & 24.77 & 0.01584 & 30.65 \\
        Conditional UOT + classifier
        & \underline{14.42} & \underline{0.00932}
        & \underline{32.18} \\
        SimNOT (ours)
        & \textbf{11.26} & \textbf{0.00867}
        & \textbf{32.60} \\
        \cmidrule(lr){1-4}
        Conditional UOT + known family
        & 14.42 & 0.00932 & 32.18 \\
        \bottomrule
    \end{tabular}
    \caption{
    Generalization to stronger degradations.
    The evaluation protocol
    follows Table~\ref{tab:shifts_milder}.
    }
    \label{tab:shifts_stronger}
\end{table*}

\subsection{Degradation Classification Results}
\label{app:classifier_results}

The degradation classifier achieves $100\%$ accuracy
on the test set with the original degradation
parameters.
Table~\ref{tab:classifier_shifts} reports its accuracy
under parameter shifts, measured against the known
degradation family.

Classification remains accurate for the tested
JPEG, blur, and noise settings.
In contrast, changing the downsampling factor
substantially affects classification for both
interpolation modes.
For bilinear factors $\times3$ and $\times5$,
no test image is classified as bilinear:
approximately $80\%$ and $85\%$ of inputs,
respectively, are classified as JPEG.
For bicubic factor $\times3$, most inputs are
classified as blur; at factor $\times5$,
all inputs are classified as blur.
This motivates reporting Conditional UOT
with both predicted and known family labels
to separate classification errors from
restoration performance.

\begin{table}[t]
    \centering
    \small
    \begin{tabular}{lcc}
        \toprule
        Degradation
        & Test parameter
        & Accuracy (\%) $\uparrow$ \\
        \midrule
        Bicubic & $\times3$ & 0.035 \\
        Bicubic & $\times5$ & 0.000 \\
        Bilinear & $\times3$ & 0.000 \\
        Bilinear & $\times5$ & 0.000 \\
        JPEG & $30$ & 100.000 \\
        JPEG & $20$ & 100.000 \\
        Gaussian blur & $1.25$ & 99.995 \\
        Gaussian blur & $1.75$ & 100.000 \\
        Gaussian noise & $0.04$ & 100.000 \\
        Gaussian noise & $0.06$ & 100.000 \\
        \bottomrule
    \end{tabular}
    \caption{
    Degradation classification accuracy under
    parameter shifts.
    Each setting contains $20{,}259$ test images.
    The classifier is trained only on the fixed
    degradation parameters specified
    in Appendix~\ref{app:celeba_details}.
    }
    \label{tab:classifier_shifts}
\end{table}

\section{Experimental Details}
\label{app:experimental_details}
\subsection{Gaussian-to-Swiss-Roll Experiment}
\label{app:swiss_roll}

\textbf{Distributions.}
We use five Gaussian source distributions
$\mathbb P_k=\mathcal N(\mu_k,I_2)$ with means
\[
\mu_1=(0,0),\quad
\mu_2=(3,0),\quad
\mu_3=(0,3),\quad
\mu_4=(-3,0),\quad
\mu_5=(0,-3).
\]
To construct the target distribution, we first sample
\[
\widetilde y=(t\cos t,t\sin t)+0.75\,\epsilon,
\qquad
t\sim\operatorname{Unif}[1.5\pi,4.5\pi],
\qquad
\epsilon\sim\mathcal N(0,I_2),
\]
where $t$ and $\epsilon$ are independent.
We then standardize each coordinate using the empirical
mean and standard deviation computed from a fixed
calibration sample of $20{,}000$ points.
These normalization statistics remain fixed throughout
training and evaluation.
Source and target samples are generated independently
on demand.

\textbf{Parametrization.}
The shared map $T_\theta:\mathbb R^2\to\mathbb R^2$
and the five potentials
$v_{\omega_k}:\mathbb R^2\to\mathbb R$
are fully connected networks with four hidden layers
of width $256$, ReLU activations, and linear output layers.
The map is deterministic and uses no residual connection.
All networks use Kaiming initialization with zero biases.

\textbf{Training.}
We use the scaled quadratic transport cost
$c(x,y)=\tau\|x-y\|_2^2$ with $\tau=10^{-3}$
and the conjugate functions
\[
\bar\psi(u)=\bar\phi(u)
=2\log(1+\exp(u))-2\log 2.
\]
Training consists of $100{,}000$ iterations,
each comprising one update of the potentials
followed by one update of the shared map.
For the potential update, we independently sample
$512$ points from each source and a common batch
of $512$ target points.
Fresh source batches of the same size are sampled
for the map update.
Contributions from the five sources are averaged
with equal weights.

We regularize the potentials using an R1 penalty
on target samples with coefficient $\gamma=5$:
\[
\frac{\gamma}{2K}\sum_{k=1}^{K}
\mathbb E_{y\sim\mathbb P^\ast}
\|\nabla_y v_{\omega_k}(y)\|_2^2,
\qquad K=5.
\]

\textbf{Visualization.}
We use fixed evaluation batches of $2{,}048$ points
from each source and $2{,}048$ target points,
sampled independently of the training batches.
Figure~\ref{fig:swiss_roll} displays the outputs
of the final map, applied directly to source samples
without rejection sampling.
The same target sample is shown in every panel.

\subsection{CelebA Restoration Experiment}
\label{app:celeba_details}

\textbf{Data split.}
We use $202{,}592$ CelebA images at resolution
$64\times64$, pooling the original dataset partitions
and constructing a new split with random seed $0$.
We allocate $91{,}166$ images to the sources,
$91{,}167$ to the clean target, and $20{,}259$
to the test set, corresponding approximately
to a $45/45/10$ split.
The source images are further divided into five
disjoint subsets: $18{,}234$ images for bicubic
downsampling and $18{,}233$ for each remaining
degradation.
Each image belongs to only one subset.
The split is performed at the image level.

Only degraded source images and independently
sampled clean target images are used for restoration
training.
The clean counterparts of source images are not
provided as supervision.
Paired degraded and clean test images are used
exclusively for evaluation.

\textbf{Degradations.}
We construct the five source distributions using
the following transformations:
\begin{itemize}
    \item \textit{Bicubic downsampling.}
    Images are resized from $64\times64$ to
    $16\times16$ and back to $64\times64$,
    using bicubic interpolation at both stages.

    \item \textit{Bilinear downsampling.}
    The same resizing procedure is applied
    using bilinear interpolation.

    \item \textit{JPEG compression.}
    Images are encoded and decoded using
    JPEG quality $25$.

    \item \textit{Gaussian blur.}
    We apply a Gaussian filter with standard
    deviation $1.5$ and kernel size $13\times13$.

    \item \textit{Gaussian noise.}
    Independent Gaussian noise with standard
    deviation $0.05$ is added to pixel values
    in $[0,1]$, followed by clipping to $[0,1]$.
\end{itemize}
Both resizing operations use
\texttt{align\_corners=False} without antialiasing.
Network inputs are normalized to $[-1,1]$.
Noise realizations are sampled during training
and fixed for evaluation.

\textbf{Architectures.}
We parametrize the shared stochastic map
$T_\theta(x,z)$ using an NCSN++ generator,
with independent auxiliary noise
$z\sim\mathcal N(0,I_{100})$.
The generator uses a base channel width of $64$,
channel multipliers $(1,1,2,2,4,4)$,
two residual blocks per resolution,
and attention at resolution $16\times16$.
The noise embedding has dimension $256$,
and the output uses a hyperbolic tangent activation.
The source index is not provided to the map.

Each of the five potentials is an independently
parametrized \texttt{Discriminator\_large} network
with base width $64$ and LeakyReLU activations
with negative slope $0.2$.

\textbf{Training.}
We use the scaled quadratic cost
$c(x,y)=\tau\|x-y\|_2^2$ with $\tau=10^{-3}$,
where the squared differences are summed over
all pixels and channels.
The conjugate functions are
\[
\bar\psi(u)=\bar\phi(u)
=2\log(1+\exp(u))-2\log 2.
\]
We train for $100{,}000$ iterations, alternating
one potential update with one map update.
Each potential update uses $16$ images from
each source and a common batch of $16$ clean
target images.
The map update uses fresh batches of $16$ images
per source.
Auxiliary noise is sampled independently for
each input, and source contributions are averaged
with equal weights.

We regularize the potentials using an R1 penalty
on clean target samples with coefficient $5$.
Adam optimizers use
$(\beta_1,\beta_2)=(0.5,0.9)$,
with initial learning rates of $2\cdot10^{-4}$
for the map and $10^{-4}$ for the potentials.
Cosine learning-rate schedulers are advanced
every $500$ training iterations, with
$T_{\max}=700$ scheduler steps and
$\eta_{\min}=10^{-5}$.
We initialize an exponential moving average
of the map parameters at iteration $30{,}000$
and subsequently update it with decay $0.999$.
All reported restoration results use the final
EMA checkpoints at iteration $100{,}000$.

\textbf{Baseline implementations.}
Pooled UOT uses one stochastic map and one potential
for the equally weighted mixture of the five sources.
Conditional UOT uses one map and one potential,
both conditioned on the degradation label.
A learned degradation embedding is added to
the generator's latent input and, for the potential,
to the features after the first convolutional layer.
The baselines use the same generator and potential
backbones, source and target split, transport cost,
and conjugate functions.
They are trained for the same number of iterations,
using $16$ source images per degradation at each step.

\textbf{Degradation classifier.}
For blind inference with Conditional UOT,
we separately train a five-class classifier
to predict the degradation family.
We split the $91{,}166$ source images into
$82{,}049$ training and $9{,}117$ validation images
using random seed $0$.
Each image is transformed using all five
training degradations, yielding $410{,}245$
training and $45{,}585$ validation examples.
The image split is performed before applying
the degradations.
The clean target subset and the $20{,}259$
test images are not used to train or select
the classifier.

The classifier is a compact ResNet-18-style network
with stage widths $(32,64,128,256)$,
two residual blocks per stage,
global average pooling, and dropout $0.1$
before the final classification layer.
It is trained from scratch for $100{,}000$ steps
using cross-entropy without label smoothing,
AdamW with learning rate $3\cdot10^{-4}$
and weight decay $10^{-4}$, and batch size $128$.
The learning rate follows a cosine schedule
with a final value of $3\cdot10^{-6}$.
We use random horizontal flips,
gradient clipping at norm $5$,
and mixed-precision training.
Degradation parameters remain fixed throughout training.

Validation is performed every $1{,}000$ steps.
We select the checkpoint with the highest validation
accuracy, using the lowest validation loss
to break ties.
The selected checkpoint is from step $63{,}000$.
At inference, the predicted class is supplied
to the conditional restoration model.

\textbf{Evaluation on training degradations.}
Each of the $20{,}259$ test images is randomly
assigned one degradation using seed $0$,
with $4{,}052$ images per degradation except
Gaussian noise, which has $4{,}051$.
This assignment is shared across methods.
We compute FID separately for the restored outputs
of each degradation.
Mean and maximum FID refer to the arithmetic mean
and maximum of these five source-specific values.
PSNR is evaluated against the corresponding
clean reference images and averaged over
all test images.

\textbf{Evaluation under parameter shifts.}
We keep all trained models fixed and change
one degradation parameter at a time.
For each setting, the corresponding degradation
is applied to all $20{,}259$ test images.
The evaluated parameters are
\[
\begin{array}{lcc}
\text{Degradation} & \text{Training} & \text{Evaluation} \\
\hline
\text{Bicubic factor} & 4 & 3,\;5 \\
\text{Bilinear factor} & 4 & 3,\;5 \\
\text{JPEG quality} & 25 & 30,\;20 \\
\text{Blur standard deviation} & 1.5 & 1.25,\;1.75 \\
\text{Noise standard deviation} & 0.05 & 0.04,\;0.06
\end{array}
\]
For downsampling factor $s$, the intermediate
resolution is
$\lfloor64/s\rfloor\times\lfloor64/s\rfloor$,
followed by resizing back to $64\times64$
with the same interpolation mode.
Thus, factors $\times3$ and $\times5$ correspond
to intermediate resolutions $21\times21$
and $12\times12$, respectively.

We report FID, PSNR, and LPIPS.
LPIPS is computed using the AlexNet backbone
in TorchMetrics with inputs in $[-1,1]$.
For stochastic restoration, evaluation uses
$100$-dimensional Gaussian latent vectors
with fixed random seed.
Conditional UOT is evaluated both with predicted
labels and with the known degradation family.
The latter supplies only the family label,
not the modified degradation parameter.

\section{Limitations and Future Research Directions}
\label{app:limitations}

\textbf{Scaling to many sources.}
SimNOT uses a shared transport map and a separate potential
for each source distribution.
While the transport network remains unchanged as the number
of sources grows, maintaining and updating the potentials
increases the training memory and computational requirements.
Sharing parameters between potentials or sampling a subset
of sources at each training step could improve scalability.
These modifications would affect training only: inference
uses a single shared map regardless of the number of sources.

\textbf{Source grouping during training.}
Although SimNOT does not require paired source-target
samples, it assumes that training samples are grouped
by source distribution, so that each sample can be
associated with the corresponding potential.
This information allows each potential to evaluate
the matching objective for its corresponding source.
Such grouping is available in our image restoration setup,
where the degradation process is known during training,
but may be incomplete or unavailable in other applications.
Extending SimNOT to partially observed source membership
is a promising direction for future research.
Once trained, the map can be applied directly to new inputs
without specifying their source.

\end{document}